\documentclass[10pt]{article}

\usepackage[margin=1in]{geometry}
\usepackage[round]{natbib}

\usepackage[utf8]{inputenc} 
\usepackage[T1]{fontenc}    
\usepackage{hyperref}       
\usepackage{url}            
\usepackage{booktabs}       
\usepackage{amsfonts}       
\usepackage{nicefrac}       
\usepackage{microtype}      
\usepackage{xcolor}         
\usepackage{amsthm}
\usepackage{amsmath}
\usepackage{amssymb}
\usepackage{float}
\usepackage{graphicx}
\usepackage{optidef}
\usepackage{subcaption}
\usepackage{tikz}
\usepackage{comment}
\usepackage{todonotes}
\usepackage{dsfont}
\usepackage{placeins}

\newtheorem{theorem}{Theorem}
\newtheorem{corollary}{Corollary}

\newtheorem{lemma}{Lemma}
\newtheorem{proposition}{Proposition}

\theoremstyle{definition}
\newtheorem{definition}{Definition}
\newtheorem{example}{Example}
\newtheorem{remark}{Remark}
\numberwithin{equation}{section}

\newcommand{\norminf}[1]{\left\|#1\right\|_{\infty}}

\newcommand{\mrm}[1]{\mathrm{#1}}

\newcommand{\R}{\mathbb{R}}

\newcommand{\N}{\mathbb{N}}

\newcommand{\da}{\downarrow}
\newcommand{\ra}{\rightarrow}

\newcommand{\cd}{\cdot}
\newcommand{\ds}{\dots}

\newcommand{\cA}{\mathcal{A}}

\newcommand{\cF}{\mathcal{F}}

\newcommand{\cP}{\mathcal{P}}

\newcommand{\cS}{\mathcal{S}}
\newcommand{\cT}{\mathcal{T}}

\newcommand{\set}[1]{\left\{{#1}\right\}}

\newcommand{\sqbkcond}[2]{\left[ #1 \middle| #2 \right]}

\definecolor{azure}{rgb}{0.0, 0.4, 0.9}

\definecolor{darkred}{rgb}{0.6, 0, 0}

\usetikzlibrary{arrows.meta, automata, positioning, calc, intersections}

\definecolor{darkblue}{HTML}{033473}
\makeatletter
\newcommand{\tikzAngleOfLine}[3]{%
  \pgfmathanglebetweenpoints{%
    \pgfpointanchor{#1}{center}}{%
    \pgfpointanchor{#2}{center}}%
  \pgfmathsetmacro{#3}{\pgfmathresult}%
}
\makeatother

\newcommand{\roundloop}[4]{%
  \node[circle,minimum size=20pt](B) at ([{shift=(#1:#2)}]#3) {};%
  \coordinate (C) at (intersection 1 of #3 and B);%
  \coordinate (D) at (intersection 2 of #3 and B);%
  \tikzAngleOfLine{B}{D}{\AngleStart}%
  \tikzAngleOfLine{B}{C}{\AngleEnd}%
  \draw[darkblue, ->, shorten >=0pt]%
   let \p1 = ($ (B) - (D) $), \n2 = {veclen(\x1,\y1)}
   in   
     (B) ++(#1:\n2) node[fill=white, scale=0.8]{#4}
     (D) arc (\AngleStart:\AngleEnd-360:\n2);%
}

\title{Near-Optimal Sample Complexities of Divergence-Based S-Rectangular Distributionally Robust Reinforcement Learning}
\author{
Zhenghao Li\\
HKUST
\and
Shengbo Wang\\
University of Southern California
\and
Nian Si\\
HKUST
}
\date{}

\begin{document}
\maketitle

\begin{abstract}
Distributionally robust reinforcement learning (DR-RL) has recently gained significant attention as a principled approach that addresses discrepancies between training and testing environments. To balance robustness, conservatism, and computational traceability, the literature has introduced DR-RL models with SA-rectangular and S-rectangular adversaries. While most existing statistical analyses focus on SA-rectangular models, owing to their algorithmic simplicity and the optimality of deterministic policies, S-rectangular models more accurately capture distributional discrepancies in many real-world applications and often yield more effective robust randomized policies. In this paper, we study the empirical value iteration algorithm for divergence-based S-rectangular DR-RL and establish near-optimal sample complexity bounds of $\widetilde{O}(|\mathcal{S}||\mathcal{A}|(1-\gamma)^{-4}\varepsilon^{-2})$, where $\varepsilon$ is the target accuracy, $|\mathcal{S}|$ and $|\mathcal{A}|$ denote the cardinalities of the state and action spaces, and $\gamma$ is the discount factor. To the best of our knowledge, these are the first sample complexity results for divergence-based S-rectangular models that achieve optimal dependence on $|\mathcal{S}|$, $|\mathcal{A}|$, and $\varepsilon$ simultaneously. We further validate this theoretical dependence through numerical experiments on a robust inventory control problem and a theoretical worst-case example, demonstrating the fast learning performance of our proposed algorithm.
\end{abstract}

\section{Introduction}
Reinforcement learning (RL) \citet{sutton2018reinforcement} is a powerful machine learning framework in which agents learn to make optimal sequential decisions through continuous interaction with an environment. While RL has achieved remarkable success across various domains, its practical deployment faces a significant challenge: real-world deployment conditions often differ from the training environment (e.g., simulations), resulting in fragile policies that fail to generalize. This mismatch undermines RL’s applicability in real-world settings, where discrepancies between training and deployment are the norm.

The framework of distributionally robust reinforcement learning (DR-RL) was thus proposed in \citet{zhou2021finite} to address this mismatch and has since been further developed in a series of works, including \citet{Panaganti2021, yang2022toward, xu2023improved, blanchet2023double_pessimism_drrl, liu22DRQ, Wang2023MLMCDRQL, yang2023avoiding, wang_sample_2024, shi_distributionally_2023}.

Popular models in distributionally robust reinforcement learning (DR-RL) include those based on SA-rectangular and S-rectangular uncertainty sets. The notion of \textit{rectangularity}, originally introduced in the robust MDP literature to describe the adversary’s temporal flexibility in selecting distributions \citep{iyengar2005robust}, has since been refined. With the incorporation of various information structures and a growing focus on constraining adversarial power, rectangularity now serves to impose structural limitations on uncertainty sets, as elaborated in \citet{le_tallec2007robustMDP} and \citet{wiesemann2013robust}. In particular, SA-rectangularity allows the adversary to choose separate distributions for each state-action pair, whereas S-rectangularity enforces consistency across actions within a given state, thereby offering a more confined modeling choice.

Existing statistical analyses of DR-RL predominantly focus on the SA-rectangular setting, primarily due to its computational tractability. Moreover, it has been shown that SA-rectangular models always admit deterministic optimal policies. However, as illustrated in the example below, the S-rectangular formulation can be more appropriate and less conservative in certain applications, such as inventory management.

\begin{example}[Inventory Model] \label{example:inventory_S_rec} 
Consider a classical inventory control problem where the inventory evolves according to
 $
 S_{t+1} = S_t + A_t - D_t,
 $
 with $\{D_t : t \geq 0\}$ representing the stochastic demand process and $A_t$ denoting the replenishment decision at time $t$. The reward function is $
R(S_t,A_t,S_{t+1}) = p(S_t-S_{t+1}+A_t)+b\min(S_{t+1},0)-h\max(S_{t+1},0)-c A_t, 
$ where $p$ is the sales price, $c$ is the purchase cost, $h$ is the holding cost, and $b$ is the penalty of backlog.  
To address the uncertainty in demand, distributionally robust reinforcement learning (DR-RL) provides a natural framework for enhancing robustness. In this context, it is reasonable to assume that the adversary can only modify the distribution of the demand $D_t$ independently of the controller’s action $A_t$, leading to an S-rectangular uncertainty set. By contrast, the SA-rectangular formulation allows the adversary to choose different distributions for $D_t$ based on the controller’s action $A_t$—for example, assigning low demand when $A_t$ is large and high demand when $A_t$ is small—granting the adversary excessive power and resulting in an unrealistic model.
\end{example}

This example highlights how S-rectangularity constrains the adversary’s power by preventing it from adapting to the controller’s actions, making it a more practical and less conservative modeling choice in applications such as inventory management.

While suitable for many applications, the S-rectangular formulation in DR-RL is more challenging than its SA-rectangular counterpart, both statistically and computationally, due to the possibility of randomized optimal policies. Computationally, this requires solving a full min-max problem rather than a simpler maximization. Fortunately, \citet{ho2018fast,ho2022robust} proposed an efficient method for performing Bellman updates in this setting. Statistically, the challenge arises from the fact that the space of randomized policies is exponentially larger than the space of deterministic policies typically sufficient under SA-rectangularity.


Another feature of Example \ref{example:inventory_S_rec} is that the reward depends on the current state $S_t$, the current action $A_t$, and the next state $S_{t+1}$. In contrast, the literature typically considers reward functions of the form $R(S_t, A_t)$, which depend only on the current state and action. The inventory management example highlights the necessity of adopting a reward function of the form $R(S_t, A_t, S_{t+1})$ to accurately capture the underlying dynamics.

In this work, we study the problem of learning the optimal value function in a divergence-based S-rectangular robust MDP, where the uncertainty set is defined as the sum of divergences across all actions. This formulation is well motivated in practice, as divergence-based uncertainty sets preserve absolute continuity and are widely adopted in the literature \citep{ho2022robust, yang2022toward}, where efficient algorithms for computing the robust value function have been developed. 

However, a satisfactory analysis of the minimax statistical complexity for learning the value function remains missing. To the best of our knowledge, the current state-of-the-art upper bound in \citet{yang2022toward} contains a sample complexity dependence on $|\mathcal{S}|$ and $|\mathcal{A}|$ in the form of $O(|\mathcal{S}|^2|\mathcal{A}|^2)$, where $|\mathcal{S}|$ and $|\mathcal{A}|$ are the cardinalities of the state and action spaces. This significantly deviates from the known lower bound of $\Omega(|\mathcal{S}||\mathcal{A}|)$. In addition, we have pointed out that in many models of practical interest (e.g., Example~\ref{example:inventory_S_rec}), the reward function depends naturally on the next state $S_{t+1}$, a structural feature that is often overlooked in the existing sample complexity literature.

We contribute to the literature by analyzing divergence-based S-rectangular robust MDPs with reward functions that depend on the current state, current action, and next state, i.e., $R(S_t, A_t, S_{t+1})$. We establish a sample complexity bound of $\widetilde O(|\mathcal{S}||\mathcal{A}|(1-\gamma)^{-4}\varepsilon^{-2})$, where $\varepsilon$ is the target accuracy and $\gamma$ is the discount factor. This bound is optimal in its dependence on $|\mathcal{S}|$, $|\mathcal{A}|$, and $\varepsilon$, and it holds uniformly over the entire range of uncertainty sizes $\rho \in (0, +\infty)$ and discount factors $\gamma \in (0,1)$. To the best of our knowledge, this is the first sample complexity upper bound for divergence-based S-rectangular models that simultaneously achieves optimal dependence on $|\mathcal{S}|$, $|\mathcal{A}|$, and $\varepsilon$.

To place our contribution in the context of the current literature, we summarize the currently available finite-sample upper and lower bounds for S-rectangular robust MDPs in Table~\ref{tab:comparison_srec_bounds}.

\begin{table}[t]
\centering
\small
\setlength{\tabcolsep}{5pt}
\begin{tabular}{llll}
\toprule
\textbf{Type} & \textbf{Ref.} & \textbf{Set} & \textbf{Sample complexity for S-rec} \\
\midrule
Upper Bound & \citet{yang2022toward} & KL &
$\widetilde{O}\!\left(\frac{|\mathcal{S}|^2|\mathcal{A}|^2}
{\varepsilon^2\rho^2 p^2(1-\gamma)^4}\right)$ \\

Upper Bound & \citet{yang2022toward} & $\chi^2$ &
$\widetilde{O}\!\left(\frac{|\mathcal{S}|^2|\mathcal{A}|^3(1+\rho)^2}
{\varepsilon^2(\sqrt{1+\rho}-1)^2(1-\gamma)^4}\right)$ \\

Upper Bound & This paper & KL, $f_k$ &
$\widetilde{O}\!\left(\frac{|\mathcal{S}||\mathcal{A}|}
{(1-\gamma)^4\mathfrak{p}_{\wedge}\varepsilon^2}\right)$ \\
\midrule
Lower Bound & \citet{yang2022toward} & $\chi^2$ &
$\widetilde{\Omega}\!\left(
\frac{|\mathcal{S}||\mathcal{A}|}{\varepsilon^2(1-\gamma)^2}
\min\!\left\{\frac{1}{1-\gamma},\frac{1}{\rho}\right\}
\right)$ \\
\bottomrule
\end{tabular}
\caption{Comparison of currently available finite-sample upper and lower bounds for S-rectangular robust MDPs. The currently available lower bound is for the $\chi^2$ ambiguity set from \citet{yang2022toward}. Our upper bounds improve the dependence on the state-action dimension for divergence-based ambiguity sets, up to logarithmic and instance-dependent factors.}
\label{tab:comparison_srec_bounds}
\end{table}

Table~\ref{tab:comparison_srec_bounds} shows that, for divergence-based S-rectangular models, our upper bounds improve the dependence on the problem dimension from the previously known quadratic-type scaling to a linear $|\mathcal S||\mathcal A|$ dependence, up to logarithmic and instance-dependent factors.

To achieve the optimal $|\mathcal{S}||\mathcal{A}|$ dependence, we develop a refined sensitivity analysis that improves upon the metric entropy bounds derived from the covering numbers of the randomized policy class $\Pi = \set{(\pi(\cdot| s))_{s \in \mathcal{S}} \mid \pi(\cdot | s) \in \Delta(\mathcal{A})}$, where $\Delta(\mathcal{A})$ denotes the probability simplex over $\mathcal{A}$, as used in \citet{yang2022toward}. Moreover, our analyses advance the techniques of \citet{wang_sample_2024} by relaxing the mutual absolute continuity requirement, thereby extending the allowable range of the uncertainty radius to $\mathbb{R}{+}$, beyond the previously restrictive regime of $\rho = O(\mathfrak{p}_\wedge)$, while retaining an $O(1)$ dependence on $\rho$ as $\rho \downarrow 0$.

The remainder of this paper is organized as follows: Section~\ref{sec:literature_review} briefly reviews related work on SA-rectangular and S-rectangular distributionally robust reinforcement learning. Section~\ref{sec:learning_s_rectangular_robust_markov_decision_processes} introduces the framework for learning S-rectangular distributionally robust Markov Decision Processes. Section~\ref{sec:sample_complexity_bounds_for_the_empirical_bellman_estimator} establishes sample complexity upper bounds for value function estimation. Section \ref{sec:numerical} presents numerical experiments to support our theoretical results. 

\section{Literature Review}
\label{sec:literature_review}
In this section, we briefly survey SA-rectangular and S-rectangular distributionally robust reinforcement learning.

\textbf{SA-rectangular DR-RL:} The dynamic programming principles for SA-rectangular distributionally robust Markov decision processes (DR-MDPs) have been gradually established through a series of works under different information structures \citep{gonzalez2002minimax, iyengar2005robust, nilim2005robust, shapiro2022distributionally, wang_foundation_2024}. Recent advances in SA-rectangular distributionally robust reinforcement learning (DR-RL) have explored sample complexity in various settings. Broadly speaking, model-based approaches have been studied in \citet{zhou21, Panaganti2021, yang2022toward, ShiChi2022, xu2023improved, shi2023curious_price, blanchet2023double_pessimism_drrl}, while the statistical properties of model-free algorithms are presented in \citet{liu22DRQ, Wang2023MLMCDRQL, wang2023VRDRQL, yang2023avoiding}.

\textbf{S-rectangular DR-RL:}  To extend the flexibility of robust MDP models, S-rectangularity was introduced in \citet{xu2010distributionally, wiesemann2013robust} as an overarching theoretical framework to constrain the adversary while retaining a dynamic programming equation. \citet{ho2018fast,ho2022robust} developed an efficient optimization algorithm to solve the Bellman update within this framework. Subsequently, \citet{kumar2024efficient} improved upon their work by proposing a faster algorithm for the $L_p$ 
 uncertainty sets. On the statistical side, \citet{yang2022toward} provided the first sample complexity result for S-rectangular DR-RL, achieving a rate of $\widetilde{O}(|\mathcal{S}|^2|\mathcal{A}|^2(1-\gamma)^{-4}\varepsilon^{-2})$, which is suboptimal in its dependence on the number of states and actions. More recently, \citet{clavier2024near} established near-optimal rates for the S-rectangular setting under general $L_p$ norm uncertainty sets. However, their analysis does not directly extend to divergence-based uncertainty sets.

\section{Learning S-rectangular Robust Markov Decision Processes}
\label{sec:learning_s_rectangular_robust_markov_decision_processes}

\subsection{Classical Markov Decision Processes} 
We briefly review and establish notation for classical tabular MDP models. Let $\Delta(\cS),\Delta(\cA)$ denote the probability simplex over the finite state space $\cS$ and action space $\cA$ respectively. An infinite horizon MDP is defined by the tuple $(\cS, \cA, R, P, \gamma)$, where $\cS$ and $\cA$ are the finite state and action spaces, respectively; $R : \cS \times \cA\times \cS \to [0,1]$ is the reward function; $P = \set{P_{s,a}(\cd)\in \Delta(\cS) : (s, a) \in \cS \times \cA}$ is the controlled transition kernel; and $\gamma \in (0,1)$ is the discount factor. Throughout the paper, given a controlled transition kernel $P$, we denote $P_s:=(P_{s,a})_{a\in\cA}$ which is seen as a function $P_s: A\ra\Delta(\cS)$.

\par We define the measurable space $(\Omega,\cF)$ to be the canonical space $(\mathcal{S}\times\mathcal{A})^{\N}$ equipped with the $\sigma$-field generated by cylinder sets. Define state-action process $(S_t,A_t)_{t\geq 0}$ by the point evaluation $S_t(\omega) = s_t,A_t(\omega) = a_t$ for all $t\geq 0$ and $\omega = (s_0,a_0,s_1,a_1,\ds)\in \Omega$. 

An agent may optimize over the class of history-dependent policies, denoted by $\Pi_{\mathrm{HD}}$, where each policy $\pi = (\pi_t)_{ t \geq 0} \in \Pi_{\mathrm{HD}}$ is a sequence of decision rules. Each decision rule $\pi_t$ at time $t$ specifies the conditional distribution of the action $A_t$ given the full history, that is, a mapping $\pi_t : (\mathcal{S} \times \mathcal{A})^t \times \mathcal{S} \to \Delta(\mathcal{A})$. In the setting of classical infinite-horizon discounted MDPs, it is well known that optimal decision-making can be achieved using stationary, Markov, deterministic policies, denoted $\Pi_{\mathrm{D}}$, where each policy is a mapping $\pi : \mathcal{S} \to \mathcal{A}$ \citep{puterman_markov_2009}.

However, in the context of S-rectangular DRMDPs, policies in $\Pi_{\mathrm{D}}$ may fail to attain the optimal performance achievable within the broader class $\Pi_{\mathrm{HD}}$ \citep{wiesemann2013robust}. In this setting, it suffices to consider stationary, Markov, randomized policies, which we denote by $\Pi$ throughout the paper. Each $\pi \in \Pi$ is a mapping $\pi : \mathcal{S} \to \Delta(\mathcal{A})$, specifying a conditional distribution over actions given the current state $S_t$, uniformly for all $t \geq 0$. Given this sufficiency, we restrict our attention to policies in the class $\Pi$ for the remainder of the paper.

\par Given a controlled transition kernel $P$ of a classical MDP, a policy $\pi\in \Pi$ and an initial distribution $\mu\in\Delta(\mathcal{S})$ uniquely defines a probability measure on $(\Omega,\cF)$. We will always assume that $\mu$ is the uniform distribution over $\mathcal{S}$. The expectation under this measure is denoted by $E_P^\pi$.  The infinite horizon discounted value $V_P^\pi$ is defined as:
$$V_P^\pi(s) := E_P^\pi\sqbkcond{\sum_{t=0}^\infty \gamma^t R(S_t, A_t,S_{t+1})}{S_0 = s}.$$ An optimal policy $\pi^*\in\Pi$ achieves the optimal value $V_P^*(s) := \max_{\pi\in \Pi} V_P^\pi(s)$.

It is well known that the optimal value function is the unique solution of the following \textit{Bellman equation}:
\begin{equation}\label{eqn:Bellman_eqn_v}
v(s) = \max_{a\in \cA}\sum_{s'\in\cS}P_{s,a}(s')(R(s,a,s') + \gamma v(s')).
\end{equation}
Let $v^*$ be the unique solution, then any deterministic policy $\pi^*: \cS\ra \cA$ with
\begin{equation*}
\pi^*(s)\in\arg\max_{a\in\cA}\sum_{s'\in\cS}P_{s,a}(s')(R(s,a,s') + \gamma v^*(s'))
\end{equation*}
will achieve the optimal value $V_P^*(s)$. 

\subsection{Robust MDPs and S-Rectangularity} 
Robust MDPs extend standard MDP models by introducing an adversary that perturbs the transition dynamics within a prescribed uncertainty set $\cP$, aiming to minimize the control value achieved by the decision maker. This formulation gives rise to a dynamic zero-sum game between the controller and the adversary. Consequently, the controller must account for potential model misspecifications represented by the adversary perturbation, leading to the design of more robust policies.

The statistical complexity of policy learning in robust MDPs has been primarily studied under SA- and S-rectangular uncertainty sets. As discussed in the previous section, S-rectangularity generalizes SA-rectangular models and provides a more expressive framework for modeling adversarial perturbations, constraining the adversary in a structured way while preserving the dynamic programming principle. From this point forward, we will be focusing on S-rectangular robust MDPs.

\begin{definition}[\citet{wiesemann2013robust}, S-rectangularity] The uncertainty set $\mathcal{P}$ is S-rectangular if
  $\mathcal{P}=\bigtimes_{s\in\mathcal{S}}\mathcal{P}_s$
for some $\mathcal{P}_s\subseteq\{(\psi_a)_{a\in\mathcal{A}}\vert \psi _a\in\Delta(\mathcal{S}),\forall a\in\mathcal{A}\}$ for all $s\in\cS$. 
\end{definition}

We focus on a special class of S-rectangular adversarial uncertainty sets, where the controlled transition kernels are perturbations of a nominal kernel $\overline{P}$. These sets are defined via a divergence function $f$ and a radius parameter $\rho$. The computational methods and statistical complexity associated with this type of uncertainty structure have been extensively studied in the literature \citep{yang2022toward, ho2018fast}.

Specifically, given a divergence function $f$, i.e. $f:\R_+\ra\R$ is convex with $f(1)=0$ and $f(0) = \lim_{t\da 0}f(t)$, we consider the S-rectangular uncertainty set $\cP(f,\rho) = \bigtimes_{s\in\cS}\cP_s(f,\rho)$ under $f$-divergence and radius $\rho$ where
\begin{equation}\label{eq:uncertainty-set}
\mathcal{P}_s(f,\rho) = \Bigg\{P_{s}\in\Delta(\mathcal{S})^{|\mathcal{A}|}\Bigg\vert P_{s,a}\ll\overline P_{s,a},\ \sum_{s^\prime\in\mathcal{S},a\in\mathcal{A}}\overline{P}_{s,a}(s')f\left(\frac{P_{s,a}(s^\prime)}{\overline P_{s,a}(s^\prime)}\right)\leq|\mathcal{A}|\rho\Bigg\}. 
\end{equation}
Here, $\ll$ denotes absolute continuity; i.e. a probability measure $p\in\Delta(\cS)$ is absolutely continuous with respect to $q\in\Delta(\cS)$, denoted by $p\ll q$, if $q(s) = 0$ implies $p(s)=0$ for any $s\in\cS$. The dependence of the uncertainty set on $(f,\rho)$ is suppressed when there is no ambiguity. 

Given a policy $\pi\in\Pi_{\mrm{HD}}$ and uncertainty set $\cP = \cP(f,\rho)$, the robust value function of $\pi$ is 
\begin{equation}\label{eqn:val_function}
V_{\cP}^\pi(s) = \inf_{P\in\cP}E_{P}^\pi \sqbkcond{\sum_{t=0}^\infty \gamma^t R(S_t, A_t,S_{t+1})}{S_0 = s}    
\end{equation}
for all $s\in \cS$. The optimal value, defined as $V_\cP^*(s) := \sup_{\pi\in \Pi_{\mrm{HD}}} V_\cP^\pi(s)$, is achieved by $\pi^*\in\Pi$. 
\begin{definition}[DR Bellman Equation] Given S-rectangular $\cP = \bigtimes_{s\in\cS} \cP_{s}$, the DR Bellman equation is the following fixed-point equation for $v:\cS\ra\R$ 
\begin{equation}\label{eqn:dr_bellman_eqn}
  v(s)=\sup_{\phi\in\Delta(\cA)}\inf_{P_s\in \cP_s}\sum_{a\in\mathcal{A}}\phi(a)\bigg[\sum_{s^\prime\in\mathcal{S}} P_{s,a}(s^\prime)\left(R(s,a,s^\prime)+\gamma v(s^\prime)\right)\bigg].
\end{equation}
It is well known \citep{wiesemann2013robust} that for $\cP = \cP(f,\rho)$ the optimal value $V^*_{\cP}$ is the unique solution $v^*$ to \eqref{eqn:dr_bellman_eqn}.
\end{definition}
We note that the value function in \eqref{eqn:val_function} assumes an adversary that fixes a controlled transition kernel over the entire control horizon, a setting commonly referred to as a static or time-homogeneous adversarial model \citep{iyengar2005robust, wiesemann2013robust, wang_foundation_2024}. This framework can be extended to more general Markovian or history-dependent adversarial models, while still preserving Markov optimality \citep{wang_foundation_2024}.

To facilitate our analysis, we define the DR Bellman operators as follows. 
\begin{definition}[DR Bellman Operators] Given uncertainty set $\cP = \cP(f,\rho)$ and $\pi\in\Pi$ the (population) DR Bellman operator is defined as
\begin{equation}\label{eqn:DR_Bellman_operator}
\mathcal{T}^\pi(v)(s):= \inf_{P\in\mathcal{P}}\Bigg(\sum_{a\in\mathcal{A}}\pi(a|s)\Bigg[\sum_{s^\prime\in\mathcal{S}} P_{s,a}(s^\prime)\left(R(s,a,s^\prime)+\gamma v(s^\prime)\right)\Bigg]\Bigg)
\end{equation}
for all $s\in \cS$. The optimal DR Bellman operator is $\mathcal{T}^*(v)(s) := \max_{\pi\in\Pi} \mathcal{T}^\pi(v)(s)$ where, for each fixed $s\in\mathcal S$, the quantity $\mathcal{T}^*(v)(s)$ depends only on the statewise action distribution $\pi(\cdot\mid s)$, and the maximum is attained by some $\pi_v^*(\cdot\mid s)\in\Delta(\mathcal A)$.
\end{definition}
\subsection{Generative Model and the Empirical Bellman Estimator}
The sample complexity analysis in this paper assumes the availability of a \textit{generative model}, a.k.a. a simulator, which allows us to sample independently from the nominal controlled transition kernel $\overline P_{s,a}$, for any $(s,a)\in \cS\times \cA$. In particular, given sample size $n$, we sample i.i.d. $\{S_{s,a}^{(1)}, \cdots, S_{s,a}^{(n)}\}$ from $\overline P_{s,a}$ and construct the empirical  transition probability
\begin{equation}\label{eqn:empirical_transition}
    \overline P_{s,a,n}(s') := \frac{1}{n}\sum_{i=1}^n\mathds 1\set{S_{s,a}^{(i)}=s'}. 
\end{equation}
Then, we define $\overline{P}_n:= \{ \overline P_{s,a,n}|(s,a)\in\cS\times\cA\}$ as the empirical nominal controlled transition kernel based on $n$ samples. We define the empirical uncertainty set $\cP_{n}(f,\rho) := \bigtimes_{s\in\cS}\cP_{s,n}(f,\rho)$ where $\cP_{s,n}(f,\rho)$ is from \eqref{eq:uncertainty-set} by replacing $\overline P_{s,a}$ with $\overline{P}_{s,a,n}$. Again, the dependence on $(f,\rho)$ will be suppressed for simplicity.

The empirical value function $V_{\hat{\mathcal{P}}}^\pi$ is defined in \eqref{eqn:val_function} with $\mathcal{P}$ replaced by $\mathcal{P}_n$. The empirical DR Bellman operator $\hat{\mathbf{T}}^\pi$ is defined as in \eqref{eqn:DR_Bellman_operator} with $\mathcal{P}$ replaced by $\mathcal{P}_n$. The corresponding optimal empirical DR Bellman operator is $\hat{\mathbf{T}}^*(v)(s) := \max_{\pi\in\Pi} \hat{\mathbf{T}}^\pi(v)(s)$, defined for each $v:S\ra\R$. 

Equipped with these definitions, we present our strategy to estimate the optimal value and policy of the S-rectangular robust MDP via the empirical value function. This is motivated by the fact that $V^*_{\cP} = v^*$ where $v^*$ solves \eqref{eqn:dr_bellman_eqn}.  
\begin{definition}[Empirical Bellman Estimators]
We define the empirical Bellman estimator $\hat v$ to $V^*_{\cP}$ as the unique fixed point of $\hat{\mathbf{T}}^*$; i.e. $\hat{\mathbf{T}}^*(\hat v) = \hat v$. Moreover, we define an empirical optimal policy to be any $\hat\pi^* \in \Pi$ that attains the maximum in $\hat{\mathbf T}^*(\hat v)$, i.e., $\hat{\mathbf T}^{\hat\pi^*}(\hat v)(s)=\hat{\mathbf T}^*(\hat v)(s)$ for all $s\in \cS$.
\end{definition}

The remainder of the paper is devoted to theoretical analysis and numerical validation of the statistical efficiency of estimating the optimal value $V^*_{\cP} = v^*$ using $\hat v$, and of estimating the optimal policy using $\hat \pi^*$. We conclude this section by introducing the following important proposition that provides an upper bound on the $l_\infty$ estimation error.


\begin{proposition}
\label{proposition:value-function-error}
Let $T$, $\hat T$ be any $\gamma$-contraction operators, and $u^*$, $\hat u$ be the solution of $T(u)=u$ and $\hat{T}(u)=u$ respectively. Then, the estimation error is upper bounded by
\begin{equation*}
\|\hat{u}-u^*\|_\infty\leq\frac{1}{1-\gamma}\left\|\hat{T}(u^*)-T(u^*)\right\|_\infty
\end{equation*}
\end{proposition}

\begin{corollary}
\label{corollary:value-function-error-1}
Let $v^*$,$\hat{v}$ be the solution of $\cT^*(v) = v$ and $\hat{\mathbf{T}}^*(v) = v$, respectively. Then, the estimation error is upper bounded by
\begin{equation*}
\|\hat{v}-v^*\|_\infty\leq\frac{1}{1-\gamma}\left\|\hat{\mathbf{T}}^*(v^*)-\mathcal{T}^*(v^*)\right\|_\infty
\end{equation*}
\end{corollary}
\begin{corollary}
\label{corollary:value-function-error-2}
Let $V_{\mathcal{P}}^{\pi}$, $V_{\hat{\mathcal{P}}}^{\pi}$ be the solution of $v=\mathcal{T}^{\pi}(v)$ and $v=\hat{\mathbf{T}}^{\pi}(v)$ respectively. Then, the estimation error is upper bounded by
\begin{equation*}
\sup_{\pi\in\Pi}\left\|V_{\hat{\mathcal{P}}}^{\pi}-V_{\mathcal{P}}^{\pi}\right\|_{\infty}\leq \frac{1}{1-\gamma}\sup_{\pi\in\Pi}\left\|\hat{\mathbf{T}}^{\pi}(V_{\mathcal{P}}^{\pi})-\mathcal{T}^{\pi}(V_{\mathcal{P}}^{\pi})\right\|_{\infty}
\end{equation*}
\end{corollary}



The proof of Proposition \ref{proposition:value-function-error} is deferred to Appendix \ref{section:value-function-error-proof}.

\section{Sample Complexity Bounds for the Empirical Bellman Estimator}
\label{sec:sample_complexity_bounds_for_the_empirical_bellman_estimator}

In this section, we establish sample complexity upper bounds to achieve an absolute $\varepsilon$ error in $l_\infty$ distance when estimating $V^*_{\cP}$ using $\hat v$.  We focus on two specific $f$-divergence uncertainty models. When $f_{\mrm{KL}}(t) = f(t) = t \log t$, the corresponding uncertainty set $\cP_s(f_{\mrm{KL}}, \rho)$ is based on the Kullback–Leibler (KL) divergence, which is widely used in the machine learning literature. Alternatively, when $f = f_k$ as defined in Definition~\ref{def:fk_divergence}, the resulting $f_k$-divergence model captures another well-studied class of uncertainty sets \citep{duchiLearningModelsUniform2021}.

We note that our analysis techniques are applicable to a broader class of smooth divergence functions $f$. However, we focus on these two representative cases for demonstration purposes. This reflects that achieving near-tight sample complexity bounds often requires leveraging specific structural properties of the divergence. In particular, we highlight the desirable feature that, in the regime where the radius $\rho\downarrow 0$, our bounds remain $O(1)$ in $\rho$, avoiding the diverging sample complexity upper bounds established in earlier results, as discussed in \citep{wang_sample_2024}.

To facilitate our analysis and establish sample complexity results, we define the minimum support probability as a complexity metric parameter as follows. 
\begin{definition}
Define the minimum support probability as
\begin{equation*}
\mathfrak{p}_{\wedge}:=\min_{s,a\in\mathcal{S}\times\mathcal{A}}\min_{s^\prime\in\mathcal{S}:\overline{P}_{s,a}(s')>0}\overline{P}_{s,a}(s')
\end{equation*}
\end{definition}
As noted in the literature, the use of $\mathfrak{p}_{\wedge}$ as a complexity metric is well justified. In the KL case, the convergence rate of the estimation error can degrade arbitrarily, depending on the specific MDP instance, if there is no lower bound on the minimum support probability. In particular, the rate can be as slow as $\Omega(n^{-1/\beta})$ for any $\beta \geq 2$ as the sample size $n$ tends to infinity \citep{si_distributionally_2020}. Similar negative results hold in the $f_k$-divergence setting when the parameter $k$ approaches 1 \citep{duchi_learning_2020}, highlighting the necessity of such a complexity measure.

\subsection{The Kullback-Leibler Divergence Uncertainty Set}
In this section, we present sample complexity results under the KL-divergence uncertainty set. Our analysis relies on the following dual representation of the DR Bellman operator and its empirical version. The strong duality underlying this representation follows from the $\phi$-divergence duality result of \citet[Sec.~3.2]{shapiro2017distributionally}.
\begin{lemma} With $\cP = \cP(f_{\mrm{KL}},\rho)$ where $f_{\mrm{KL}}(t) = t\log t$ and $\rho\in(0,\infty)$, for any $\pi\in\Pi$ and $s\in\mathcal{S}$, the dual form of the DR Bellman operator with KL uncertainty set $\cP$ is 
  \begin{equation}
  \mathcal{T}^\pi(v)(s) = \sup_{\lambda\geq 0}\Bigg(-\lambda|\mathcal{A}|\rho-\sum_{a\in\mathcal A}\lambda\log \mathbb E_{\overline{P}_{s,a}}\left[\exp\left(-\frac{\pi(a\vert s)w(s,a,S)}{\lambda}\right)\right]\Bigg).
  \label{eq: population-DR-Bellman-operator}
  \end{equation}
where $w(s,a,S)=R(s,a,S)+\gamma v(S)$. 
The KL empirical DR Bellman operator $\hat{\mathbf{T}}^{\pi}$ satisfies \eqref{eq: population-DR-Bellman-operator} with $\overline{P}_{s,a}$ replaced by $\overline{P}_{s,a,n}$. 
  \label{lemma:kl-divergence-bellman-operator}
\end{lemma}
The proof of Lemma~\ref{lemma:kl-divergence-bellman-operator} is provided in Appendix~\ref{sub:kl-dual-proof}. Building on this dual formulation, we next analyze the statistical error between the empirical and population DR Bellman operators.

\begin{proposition}
Under the KL-divergence uncertainty set with any $\rho\in(0,\infty)$, for any $v:\cS\ra \R$ and $n\geq12\mathfrak{p}_\wedge^{-1}\log(4|\mathcal{S}|^2|\mathcal{A}|/\eta)$, with probability at least $1-\eta$,
\begin{equation*}
  \|\hat{\mathbf{T}}^*(v)-\mathcal{T}^*(v)\|_\infty\leq\sup_{\pi\in\Pi}\|\hat{\mathbf{T}}^{\pi}(v)-\mathcal{T}^{\pi}(v)\|_\infty\leq\frac{9\|R+\gamma v\|_{\infty}}{\sqrt{n\mathfrak{p}_{\wedge}}}\sqrt{\log\left(4|\mathcal{S}|^2|\mathcal{A}|/\eta\right)}. 
\end{equation*}
\label{proposition:kl-bellman-operator-error}
\end{proposition}
The proof of Proposition~\ref{proposition:kl-bellman-operator-error} is provided in Appendix~\ref{sub:proof_of_proposition_kl}. Then, combining Proposition~\ref{proposition:kl-bellman-operator-error} with Proposition \ref{proposition:value-function-error}, and the fact that $\norminf{R+\gamma v^*}\leq 1/(1-\gamma)$ under our assumption that $R\in[0,1]$, we arrive at the following theorem. The proof is presented in Appendix \ref{sub:proof_of_theorem_kl}. 
\begin{theorem}
Assume the adversary chooses from a KL-divergence uncertainty set with $\rho\in(0,\infty)$. If the sample size $n\geq12\mathfrak{p}_\wedge^{-1}\log(4|\mathcal{S}|^2|\mathcal{A}|/\eta)$, then, with probability at least $1-\eta$,
\begin{equation*}
    \|\hat{v} - v^*\|_\infty\leq\frac{9}{(1-\gamma)^2\sqrt{n\mathfrak{p}_\wedge}}\sqrt{\log(4|\cS|^2|\cA|/\eta)}.
\end{equation*}
Moreover, with probability at least $1-\eta$,
\begin{equation*}
  \sup_{\pi\in\Pi}V_{\mathcal{P}}^{\pi}(s)-V_{\mathcal{P}}^{\hat{\pi}^*}(s)\leq\frac{18}{(1-\gamma)^2\sqrt{n\mathfrak{p}_\wedge}}\sqrt{\log{(4|\mathcal{S}|^2|\mathcal{A}|/\eta)}}.
\end{equation*}
\label{theorem: kl}
\end{theorem}

\begin{remark}
Therefore, under the KL-divergence, a total of $\widetilde O(|\cS||\cA|(1-\gamma)^{-4}\mathfrak{p}_\wedge^{-1}\varepsilon^{-2})$ samples from the simulator suffices to obtain an $\varepsilon$-accurate estimate of $v^*$ with $\hat v$, or an $\varepsilon$-optimal policy with high probability. 
\end{remark}

\subsection[fk-Divergence Uncertainty Set]{$f_k$-Divergence Uncertainty Set}
Next, we consider a subclass of the Cressie-Read family of $f_k$-divergence with $k\in(1,\infty)$, as studied in \citet{duchiLearningModelsUniform2021}. 

\begin{definition} \label{def:fk_divergence}
For $k\in(1,\infty)$, the $f_k$-divergence is defined by the divergence functions $f_k(t):=(t^k-kt+k-1)/(k(k-1))$. We also define $k^* = k/(k-1)$. 
\end{definition}
Notably, when $k = 2$, the $f_2$-divergence is the $\chi^2$-divergence, which sees extensive application in the statistical testing literature. Moreover, when $k\da 1$, the $f_k$ induced divergence converges to KL. 

The analysis for $f_k$-divergence uncertainty sets follows the same strategy to KL-divergence in the previous subsection. Below we summarise the main results.

\begin{lemma}
With $\cP = \cP(f_k,\rho)$ and $\rho\in(0,\infty)$, for any $\pi\in\Pi$ and $s\in\mathcal{S}$, the dual form of the DR Bellman operator with $f_k$ uncertainty set $\cP$ is
\begin{equation*}
\mathcal{T}^\pi (v)(s) = \sup_{\eta\in\mathbb{R}^{|\mathcal{A}|}}\Bigg[\sum_{a\in\mathcal{A}} \eta_a-c\left(\sum_{a\in\mathcal{A}}\mathbb{E}_{\overline{P}_{s,a}}\left[(\eta_a-\pi(a|s)w(s,a,S))_+^{k^*} \right]\right)^{1/k^*}\Bigg]
\end{equation*}
where $c = c(k,\rho,|\mathcal{A}|) = |\mathcal{A}|^{1/k}\left(k(k-1)\rho+1 \right)^{1/k}$, $(\cd)_+ = \max(\cd,0)$ and $w(s,a,S)=R(s,a,S)+\gamma v(S)$. The $f_k$ empirical DR Bellman operator $\hat{\mathbf{T}}^{\pi}$ satisfies a similar equality with $\overline{P}_{s,a}$ replaced by $\overline{P}_{s,a,n}$.
\label{lemma:f-divergence-bellman-operator}
\end{lemma}
The proof of Lemma~\ref{lemma:f-divergence-bellman-operator} is provided in Appendix~\ref{sub:f-dual-proof}. Again, with this dual representation of the DR Bellman operators and refined estimation error analysis, we arrive at the following result. 

\begin{proposition}
\label{proposition:f-divergence-bellman-operator-error}
Under the $f_k$-divergence uncertainty set with any $\rho\in(0,\infty)$, for any $v:\cS\ra \R_+$ and $n\geq12\mathfrak{p}_\wedge^{-1}\log(4|\mathcal{S}|^2|\mathcal{A}|/\eta)$, w.p. at least $1-\eta$,
\begin{equation*}
  \|\hat{\mathbf{T}}^*(v)-\mathcal{T}^*(v)\|_\infty\leq\sup_{\pi\in\Pi}\|\hat{\mathbf{T}}^{\pi}(v)-\mathcal{T}^{\pi}(v)\|_\infty\leq\frac{9\|R+\gamma v\|_{\infty}}{\sqrt{n\mathfrak{p}_{\wedge}}}\sqrt{\log\left(4|\mathcal{S}|^2|\mathcal{A}|/\eta\right)}. 
\end{equation*}
\end{proposition}
The proof of Proposition~\ref{proposition:f-divergence-bellman-operator-error} is provided in Appendix~\ref{sub:proposition-f-proof}. This, combined with Proposition \ref{proposition:value-function-error}, implies the following error bound, whose proof is deferred to Appendix~\ref{sub:theorem-f-proof}.

\begin{theorem}
Assume the adversary chooses from a $f_k$-divergence uncertainty set with $\rho\in(0,\infty)$. If $n\geq 12\mathfrak{p}_\wedge^{-1}\log(4|\mathcal{S}|^2|\mathcal{A}|/\eta)$, then, w.p. at least $1-\eta$,
\begin{equation*}
    \|\hat{v}-v^*\|_\infty\leq \frac{9}{(1-\gamma)^2\sqrt{n\mathfrak{p}_\wedge}}\sqrt{\log(4|\cS|^2|\cA|/\eta)}.
\end{equation*}
Moreover, w.p. at least $1-\eta$,
\begin{equation*}
\sup_{\pi\in\Pi}V_{\mathcal{P}}^{\pi}(s)-V_{\mathcal{P}}^{\hat{\pi}^*}(s)\leq\frac{18}{(1-\gamma)^2\sqrt{n\mathfrak{p}_\wedge}}\sqrt{\log\left(4|\mathcal{S}|^2|\mathcal{A}|/\eta\right)}.
\end{equation*}
\label{theorem:f}
\end{theorem}
\begin{remark}
Therefore, under the $f_k$-divergence, an $\varepsilon$-accurate estimate of $v^*$ with $\hat v$, or an $\varepsilon$-optimal policy with high probability, can be obtained using a total of $\widetilde O(|\cS||\cA|(1-\gamma)^{-4}\mathfrak{p}_\wedge^{-1}\varepsilon^{-2})$ samples.

\end{remark}

\begin{figure}[ht]
  \centering
  \begin{subfigure}{0.35\linewidth}
    \includegraphics[width=\linewidth]{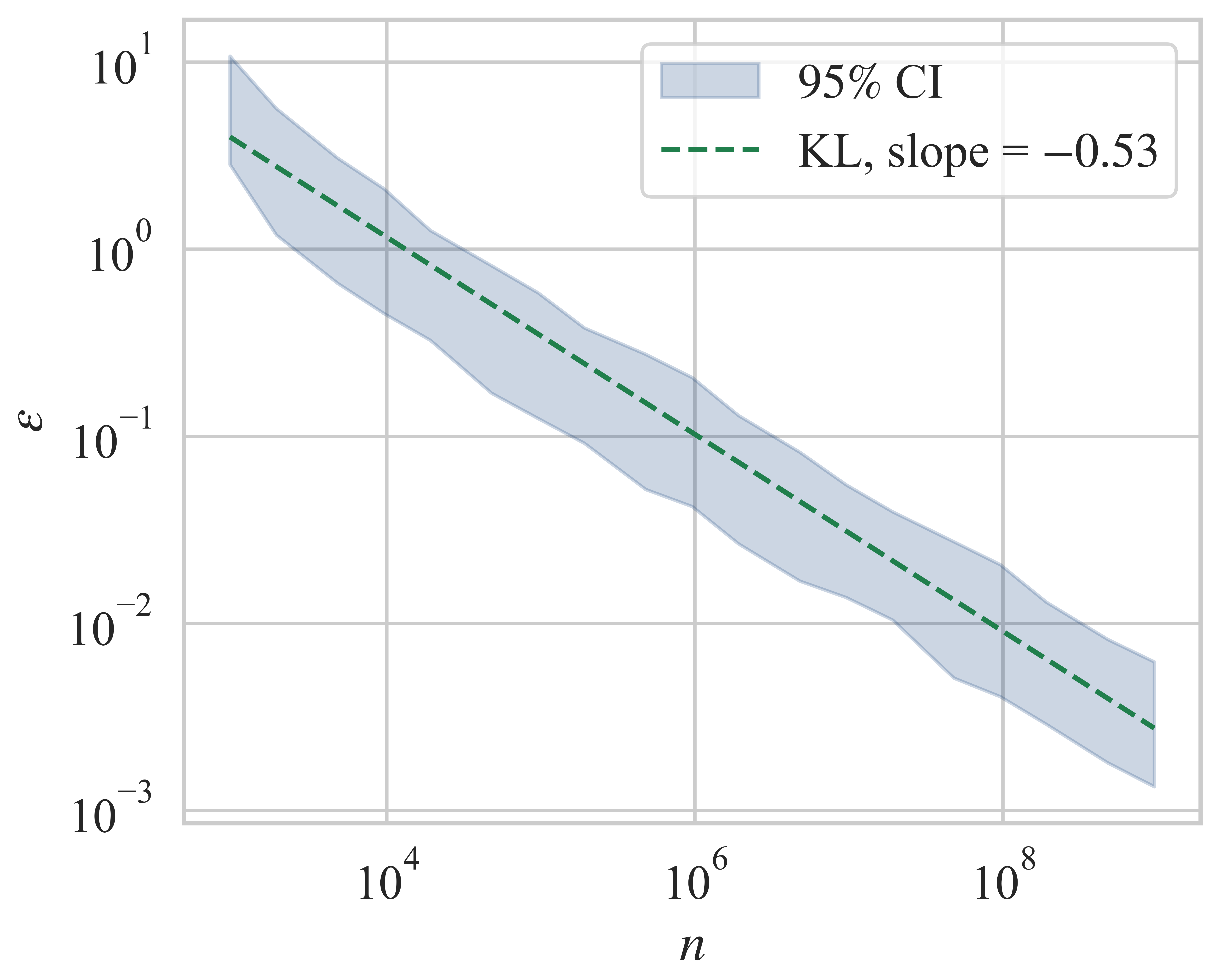}
    \caption{Uncertainty sets based on KL-divergence}
    \label{fig:n-KL}
  \end{subfigure}
  \begin{subfigure}{0.35\linewidth}
    \includegraphics[width=\linewidth]{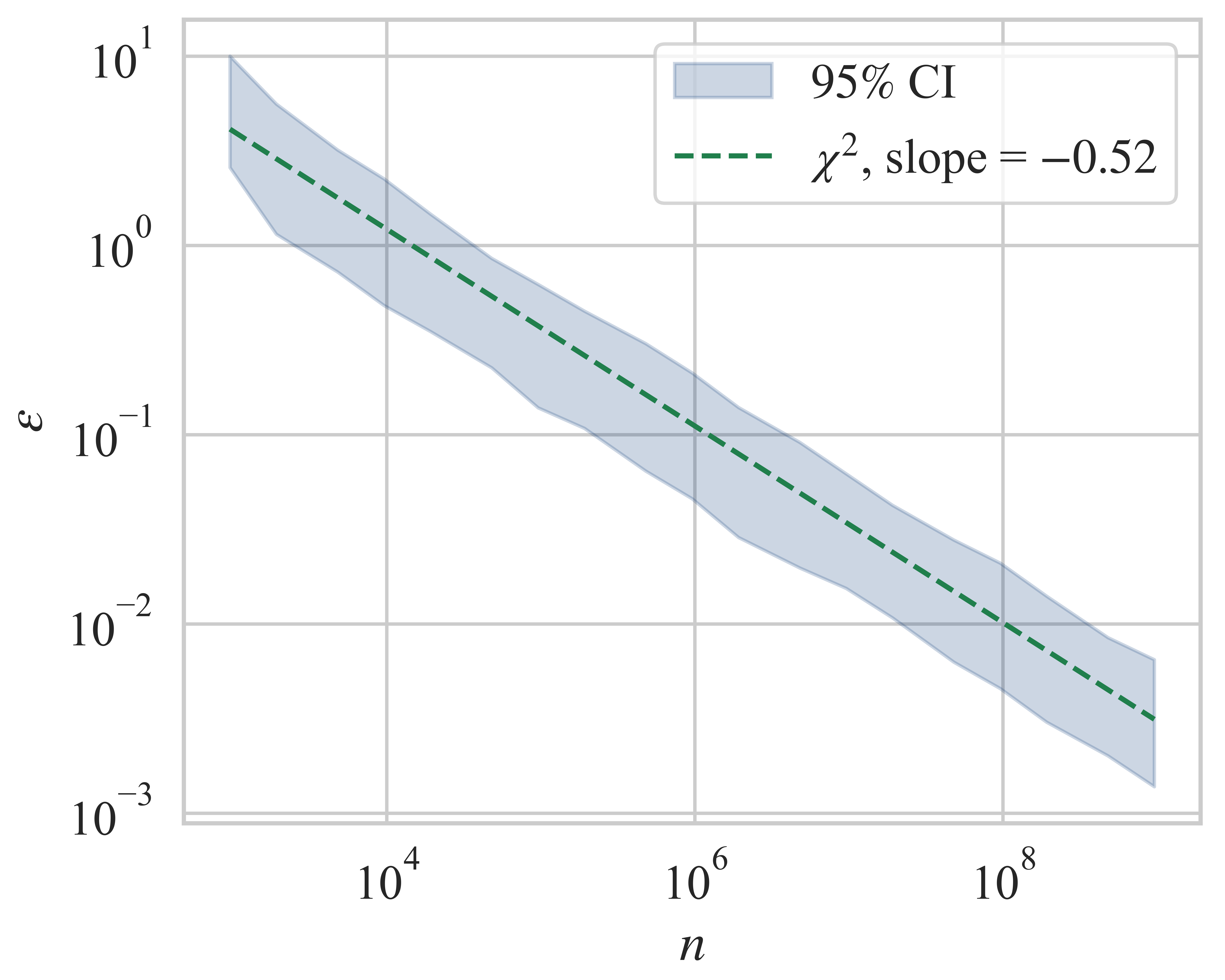}
    \caption{Uncertainty sets based on $\chi^2$-divergence}
    \label{fig:n-chi-square}
  \end{subfigure}
  \caption{Estimation error versus sample size $n$ in the robust inventory control problem.}
  \label{fig:n}
\end{figure}

\section{Numerical Experiments}
\label{sec:numerical}
In this section, we present two sets of numerical examples. In Section \ref{subsec:inventory}, we revisit the robust inventory problem from \citet{ho2018fast}, which features uncertain demand, to demonstrate the $n^{-1/2}$ error decay rate. In Section \ref{subsec:ld_example}, we consider an example from \citet{yang2022toward} to illustrate the linear dependence on $|\mathcal{S}||\mathcal{A}|$, which matches the lower bound established in \citet{yang2022toward}.
\subsection{Robust Inventory Control Problems}
\label{subsec:inventory}

We investigate the dependency of the estimation error $\varepsilon$ on the sample size $n$ and evaluate our approach on a classical discrete-time inventory management problem with stochastic demand and backlog \citep{ho2018fast}. In each period $t$, an agent decides the order quantity to maximize cumulative discounted rewards, accounting for holding costs, backlog penalties, and profits.

Let $I$ denote the maximum inventory level, $B$ the maximum backlog, and $O$ the maximum order quantity per period. The state space is defined as $\mathcal{S}=\{-B,\cdots, 0,\cdots,I\}$, the action space is $\mathcal{A} = \{0,\cdots,O\}$, where $s_t\in\mathcal{S}$ and $a_t\in\mathcal{A}$ denote the inventory and order item at the beginning of period $t$. Demand $D_t\in\{0,\cdots,D_{\max}\}$ is an i.i.d. sequence, with distribution $P_D\in\Delta(D_{\max}+1)$.


The MDP dynamics proceed as follows. Due to storage constraints, the effective order size is $\tilde{A}_t = \min(A_t, I-S_t).$
Then, the next state evolves as $S_{t+1} = \max(S_t + \tilde{A}_t - D_t, -B),$ ensuring that the backlog does not exceed $B$. The actual sales in period $t$ are given by
$X_t = S_t - S_{t+1} + \tilde{A}_t.$
A one-step reward $R(S_t,A_t,S_{t+1}) = pX_t+b\min(S_{t+1},0)-h\max(S_{t+1},0)-c\tilde A_t$ is collected, 
where $p$ is the sales price, $c$ is the purchase cost, $h$ is the holding cost, and $b$ is the penalty of backlog.

\begin{figure}[!ht]
\centering
\begin{tikzpicture}[
  ->,
  shorten >=1pt,
  node distance=2cm,
  on grid,
  >={Stealth[round]},
  every state/.style={draw=black,thick,minimum size=20pt,inner sep=0pt},
  every dot/.style={},
  every loop/.style={looseness=5}
]
  \node at (0.7,1) {$\mathcal{Y}_1$};
  \node at (0.7,2.5) {$\mathcal{Y}_2$};
  \node[state] (x_1)  at (3.5 ,0  ) {$x_1$};
  \node[state] (y_11) at (1.5 ,1  ) {};
  \node[state] (y_12) at (3   ,1  ) {};
  \node[state] (y_13) at (5   ,1  ) {};
  \node[state] (y_21) at (1.5 ,2.5) {};
  \node[state] (y_22) at (3   ,2.5) {};
  \node[state] (y_23) at (5   ,2.5) {};
  \node[state] (x_2)  at (7.5 ,0  ) {$x_2$};
  \node[state] (y_14) at (6.5 ,1  ) {};
  \node[state] (y_15) at (8.5 ,1  ) {};
  \node[state] (y_24) at (6.5 ,2.5) {};
  \node[state] (y_25) at (8.5 ,2.5) {};
  \node[state] (x_3)  at (11.5,0  ) {$x_{|\mathcal{S}|}$};
  \node[state] (y_16) at (10.5,1  ) {};
  \node[state] (y_17) at (12.5,1  ) {};
  \node[state] (y_26) at (10.5,2.5) {};
  \node[state] (y_27) at (12.5,2.5) {};

  \node[draw=black,fill=black,minimum size=2pt,inner sep=0pt,circle] (point_1) at (3.7,1.8) {};
  \node[draw=black,fill=black,minimum size=2pt,inner sep=0pt,circle] (point_2) at (4,1.8) {};
  \node[draw=black,fill=black,minimum size=2pt,inner sep=0pt,circle] (point_3) at (4.3,1.8) {};
  \node[draw=black,fill=black,minimum size=2pt,inner sep=0pt,circle] (point_1) at (7.2,1.8) {};
  \node[draw=black,fill=black,minimum size=2pt,inner sep=0pt,circle] (point_2) at (7.5,1.8) {};
  \node[draw=black,fill=black,minimum size=2pt,inner sep=0pt,circle] (point_3) at (7.8,1.8) {};
  \node[draw=black,fill=black,minimum size=2pt,inner sep=0pt,circle] (point_1) at (11.2,1.8) {};
  \node[draw=black,fill=black,minimum size=2pt,inner sep=0pt,circle] (point_2) at (11.5,1.8) {};
  \node[draw=black,fill=black,minimum size=2pt,inner sep=0pt,circle] (point_3) at (11.8,1.8) {};
  \node[draw=black,fill=black,minimum size=2pt,inner sep=0pt,circle] (point_1) at (9.2,0) {};
  \node[draw=black,fill=black,minimum size=2pt,inner sep=0pt,circle] (point_2) at (9.5,0) {};
  \node[draw=black,fill=black,minimum size=2pt,inner sep=0pt,circle] (point_3) at (9.8,0) {};

  \path[]
  (x_1) edge node [below left] {$a_1$} (y_11.south)
  (x_1) edge node [right] {$a_2$} (y_12.south)
  (x_1) edge node [below right] {$a_{|\mathcal{A}|}$} (y_13.south)
  (x_2) edge node [below left] {$a_1$} (y_14.south)
  (x_2) edge node [below right] {$a_{|\mathcal{A}|}$} (y_15.south)
  (x_3) edge node [below left] {$a_1$} (y_16.south)
  (x_3) edge node [below right] {$a_{|\mathcal{A}|}$} (y_17.south)

  (y_11) edge node[scale=0.8, below right] {$1-p$} (y_21)
  (y_12) edge node[scale=0.8, below right] {$1-p$} (y_22)
  (y_13) edge node[scale=0.8, below right] {$1-p$} (y_23)
  (y_14) edge node[scale=0.8, below right] {$1-p$} (y_24)
  (y_15) edge node[scale=0.8, below right] {$1-p$} (y_25)
  (y_16) edge node[scale=0.8, below right] {$1-p$} (y_26)
  (y_17) edge node[scale=0.8, below right] {$1-p$} (y_27)

  ;
  \roundloop{90}{15pt}{y_21}{$1$}
  \roundloop{90}{15pt}{y_22}{$1$}
  \roundloop{90}{15pt}{y_23}{$1$}
  \roundloop{90}{15pt}{y_24}{$1$}
  \roundloop{90}{15pt}{y_25}{$1$}
  \roundloop{90}{15pt}{y_26}{$1$}
  \roundloop{90}{15pt}{y_27}{$1$}
  \roundloop{0}{15pt}{y_11}{$p$}
  \roundloop{0}{15pt}{y_12}{$p$}
  \roundloop{0}{15pt}{y_13}{$p$}
  \roundloop{0}{15pt}{y_14}{$p$}
  \roundloop{0}{15pt}{y_15}{$p$}
  \roundloop{0}{10pt}{y_16}{$p$}
  \roundloop{0}{10pt}{y_17}{$p$}
\end{tikzpicture}
\caption{MDP instances from the lower bound construction in \citet{yang2022toward}.}
\label{fig:general-mdp}
\end{figure}
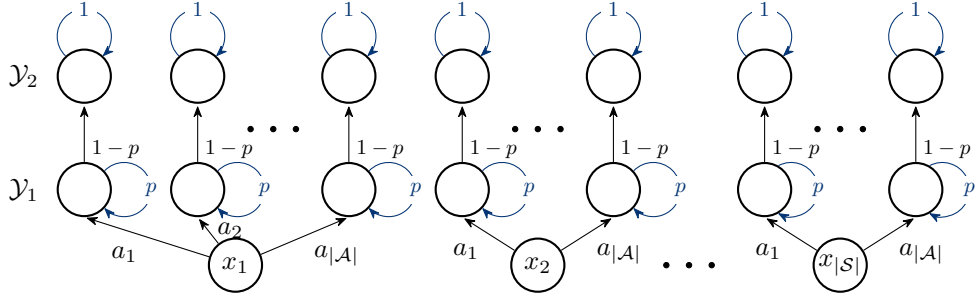

For our experiments, we set the parameters as follows: $I = 10$, $B = 5$, $O = 5$, $p = 3$, $c = 2$, $h = 0.2$, $b = 3$, $\gamma = 0.9$, and use the nominal demand distribution $P_D = [0.1, 0.2, 0.3, 0.3, 0.1]$, supported on $\{0, 1, 2, 3, 4\}$. For each $(s,a)$, we sample $n_0$ samples from the nominal transition kernels to generate the estimated transition probability $P_n$, and solve the DR-MDP problem with uncertainty size $\rho=1/6$ using the algorithm presented in \cite{ho2022robust}.

Figure \ref{fig:n} illustrates the relationship between the sample size $n$ and the error $\varepsilon$ between the empirical and population value functions. As shown in the log-log plot, the slope is approximately $-0.5$ for both the KL and $\chi^2$ cases, indicating that the error decreases at a rate proportional to $1/\sqrt{n}$.

\subsection[MDP Instances from the Lower Bound Construction]{MDP Instances from the Lower Bound Construction in \citet{yang2022toward}}
\label{subsec:ld_example}

In this section, we investigate the relationship between the estimation error and the sizes of the state space $|\mathcal{S}|$ and action space $|\mathcal{A}|$. We adopt the classic MDP structure introduced in \citet{gheshlaghi2013minimax} and \citet{yang2022toward}, which comprises three subsets: $\mathcal{S}$, $\mathcal{Y}_1$, and $\mathcal{Y}_2$, as illustrated in Figure~\ref{fig:general-mdp}.

Specifically, $\mathcal{S}$ denotes the set of all initial states, each associated with an action set $\mathcal{A}$. When an action $a_i \in \mathcal{A}$ is taken in state $s \in \mathcal{S}$, the system deterministically transitions (with probability 1) to the corresponding state $y_{1,s,a} \in \mathcal{Y}_1$. From each $y_{1,s,a}$, the system either remains in the same state with nominal probability $p$, or transitions to the corresponding absorbing state $y_{2,s,a} \in \mathcal{Y}_2$ with nominal probability $1-p$. All states in $\mathcal{Y}_2$ are absorbing, meaning that once the system enters one of these states, it remains there indefinitely via a self-loop with probability 1. 
The reward function is defined such that a reward of 1 is obtained only when the system is in any state within $\mathcal{Y}_1$; all other states yield a reward of 0. We solve the DR-MDP problem with $\gamma=0.9$ and uncertainty size $\rho = 0.1$.

\begin{figure}[!ht]
  \centering
  \begin{subfigure}{0.48\linewidth}
    \includegraphics[width=\linewidth]{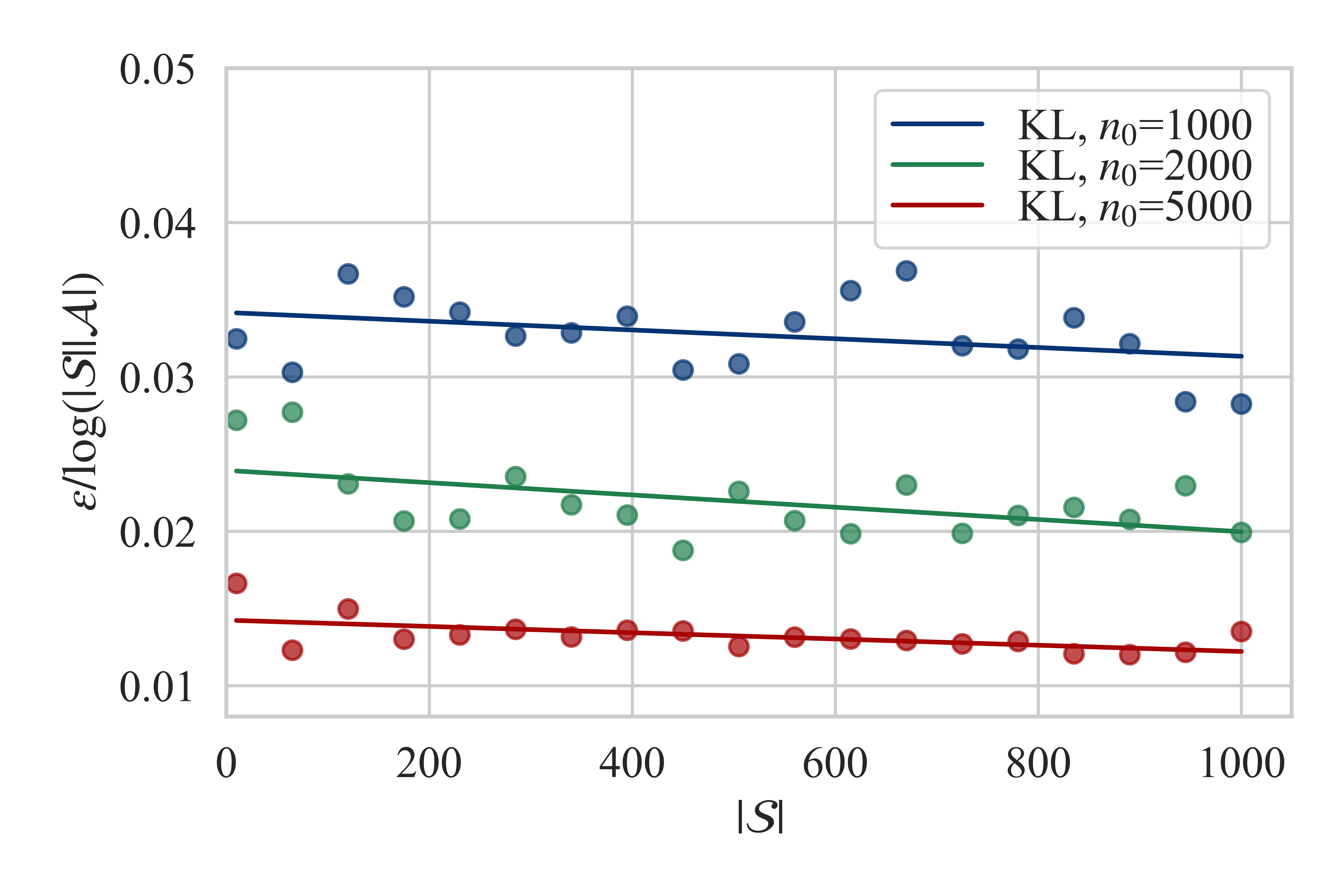}
    \caption{Uncertainty sets based on KL-divergence}
    \label{fig:general-mdp-S-a}
  \end{subfigure}
  \begin{subfigure}{0.48\linewidth}
    \includegraphics[width=\linewidth]{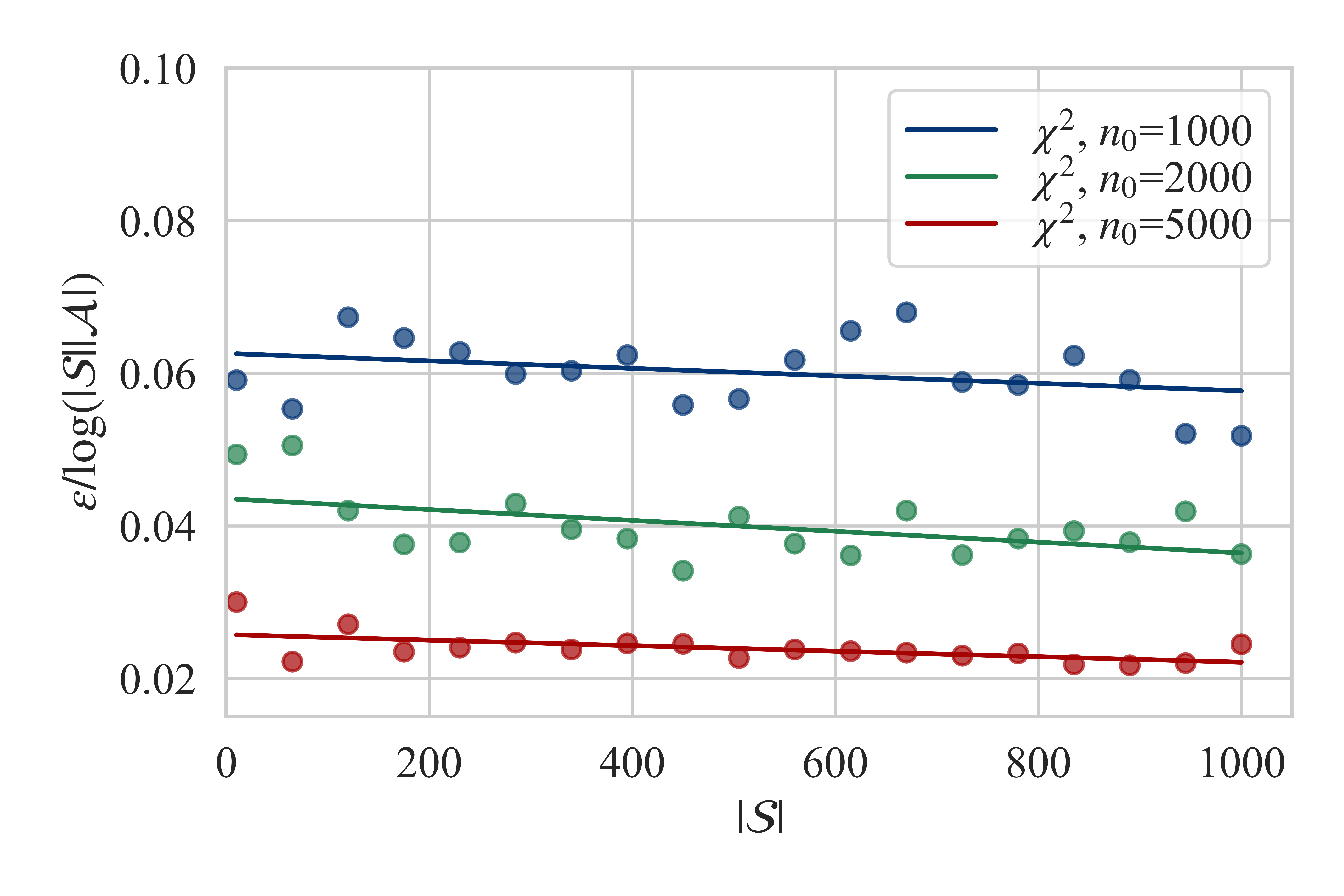}
    \caption{Uncertainty sets based on $\chi^2$-divergence}
    \label{fig:general-mdp-S-b}
  \end{subfigure}
 \caption{Estimation error versus the number of states $|\mathcal{S}|$ for the MDP instances based on the lower bound construction in \citet{yang2022toward}.}
  \label{fig:general-mdp-S}
\end{figure}

In our experiments, we first fix $|\mathcal{A}| = 65$ and vary the number of states from 10 to 1000, with the results shown in Figure~\ref{fig:general-mdp-S}. We then fix $|\mathcal{S}| = 65$ and vary $|\mathcal{A}|$ over the same range, with the corresponding results presented in Figure~\ref{fig:general-mdp-A}.

To align with our theoretical results, we normalize the estimation error by dividing it by $\log(|\mathcal{S}||\mathcal{A}|)$. Figures~\ref{fig:general-mdp-S} and~\ref{fig:general-mdp-A} display the behavior of this normalized error as $|\mathcal{S}|$ and $|\mathcal{A}|$ vary, respectively. Specifically, for each $(s, a)$ pair, we use $n_0$ samples, resulting in a total of $n_0 |\mathcal{S}| |\mathcal{A}|$ samples. The left subfigures correspond to the KL-divergence case, while the right subfigures correspond to the $\chi^2$-divergence case. 
We observe that as either $|\mathcal{S}|$ or $|\mathcal{A}|$ increases, the normalized error is non-increasing. This is consistent with our theoretical analysis, which predicts that the sample complexity scales linearly (up to logarithmic factors) with the product $|\mathcal{S}||\mathcal{A}|$.


\begin{figure}[!ht]
  \centering
  \begin{subfigure}{0.48\linewidth}
    \includegraphics[width=\linewidth]{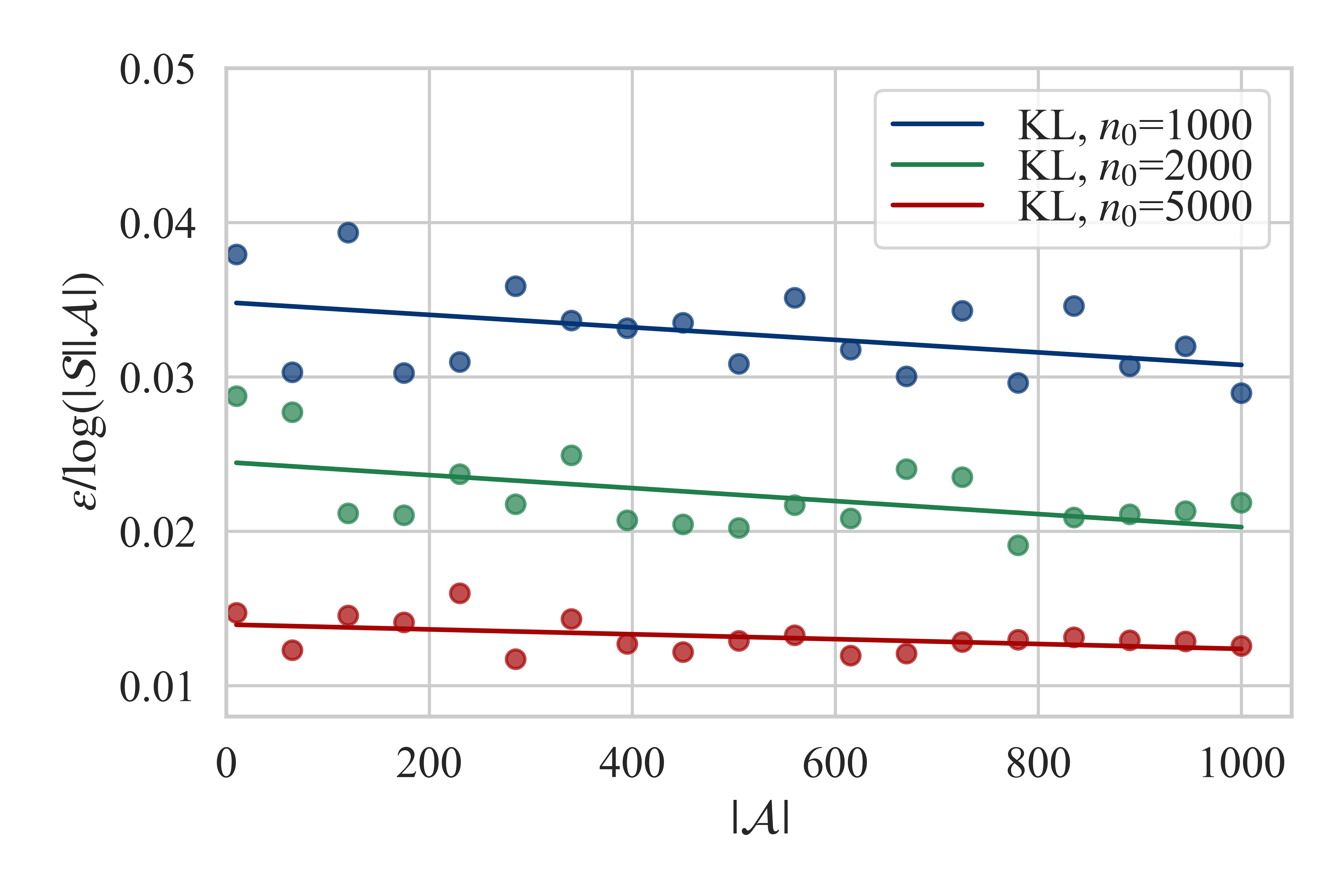}
    \caption{Uncertainty sets based on KL-divergence}
    \label{fig:general-mdp-A-a}
  \end{subfigure}
  \begin{subfigure}{0.48\linewidth}
    \includegraphics[width=\linewidth]{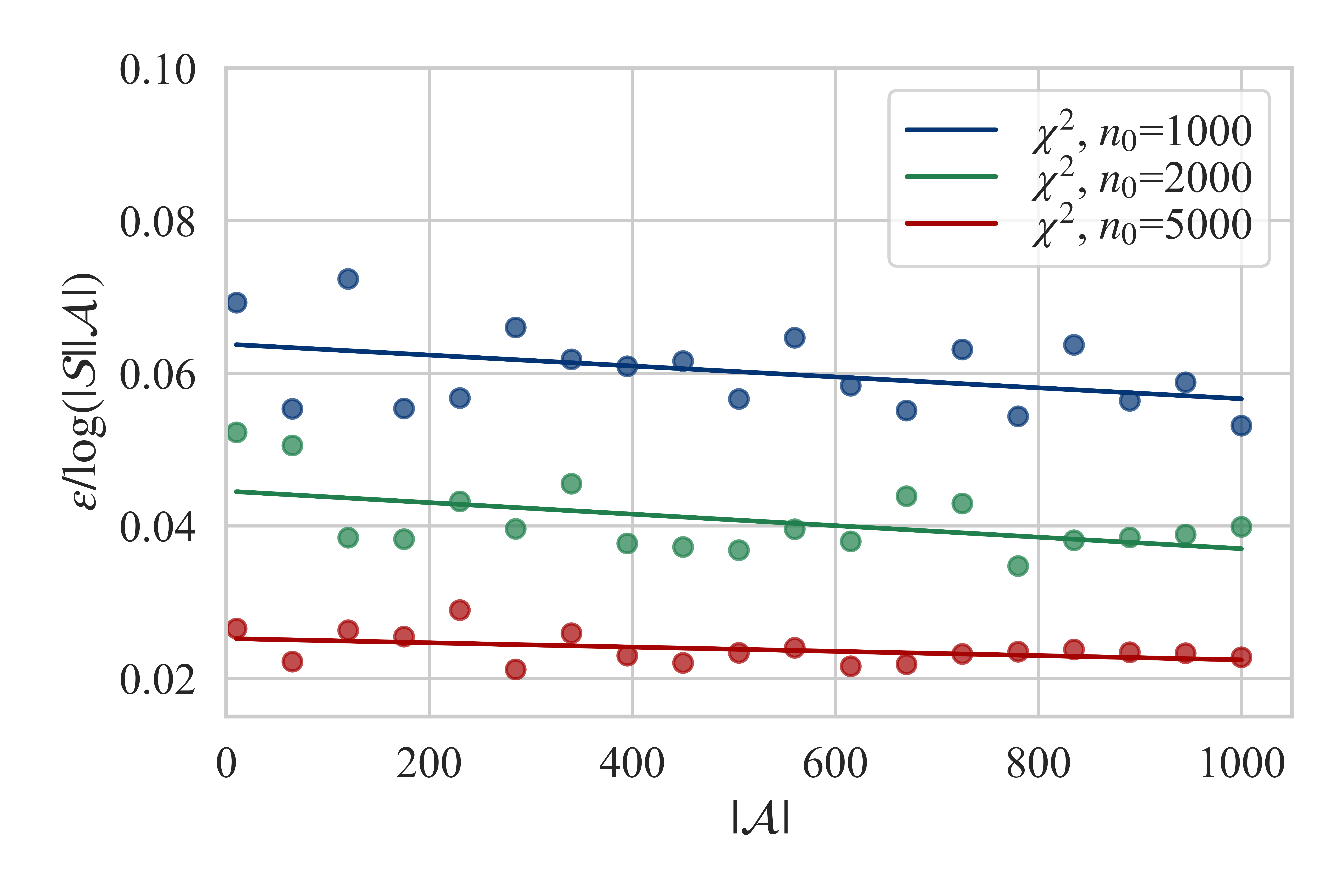}
    \caption{Uncertainty sets based on $\chi^2$-divergence}
    \label{fig:general-mdp-A-b}
  \end{subfigure}
  \caption{Estimation error versus the number of actions $|\mathcal{A}|$ for the MDP instances based on the lower bound construction in \citet{yang2022toward}.}
  \label{fig:general-mdp-A}
\end{figure}

\section{Conclusion and Future Work}
\label{sec:conclusion_and_future_work}
In this paper, we establish near-optimal sample complexity results for divergence-based S-rectangular robust MDPs under the discounted reward criterion. To the best of our knowledge, our results are the first to achieve optimal dependence on $|\mathcal{S}|$, $|\mathcal{A}|$, and $\varepsilon$ simultaneously.  
Looking ahead, several promising research directions remain open. A natural extension is to move beyond the generative model assumption and develop provable guarantees in more practical settings, such as model-free reinforcement learning and offline reinforcement learning. Another important avenue is to explore richer formulations of uncertainty, for instance, by modeling uncertainty around the underlying noise process in distributionally robust stochastic control. Finally, it is also interesting to investigate S-rectangular robust MDPs under the average-reward criterion.
\section*{Acknowledgement}
N. Si’s research is supported in part by the Early Career Scheme [Grant 	26210125] from the Hong Kong Research Grants Council and HKUST Li \& Fung Supply Chain Institute Research Grant 2025.
\FloatBarrier
\bibliographystyle{plainnat}
\bibliography{references,DR_MDP}

\clearpage
\appendix
\section*{Supplementary Materials}
\addcontentsline{toc}{section}{Supplementary Materials}

\section{Proofs of Value Function Error Bounds}
\label{section:value-function-error-proof}
In this section, we prove Proposition \ref{proposition:value-function-error}, Corollary \ref{corollary:value-function-error-1} and Corollary \ref{corollary:value-function-error-2}. We first show that both the population and the empirical S-rectangular Bellman operators $\mathcal{T}^*$ and $\hat{\mathbf{T}}^*$ are $\gamma$-contractions. Furthermore, for given $\pi$, the Bellman operators $\mathcal{T}^\pi$ and $\hat{\mathbf{T}}^\pi$ are also $\gamma$-contractions. This is a well-known fact, see for example \citep{wang_foundation_2024}. We include a proof to make the paper self-contained. 

\begin{lemma}
\label{lemma:Lipschitz-continuity}
Let
  \begin{equation*}
  f_{\pi,P}(v)(s) = \sum_{a\in\mathcal{A}}\pi(a|s)\left[\sum_{s^\prime\in\mathcal{S}} P_{s,a}(s^\prime)\left(R(s,a,s^\prime)+\gamma v(s^\prime)\right)\right]
  \end{equation*}
Then,
\begin{equation*}
\sup_{\pi\in\Pi}\sup_{P\in\mathcal{P}}\left| f_{\pi,P}(v_1)(s)-f_{\pi,P}(v_2)(s)\right|\leq\gamma \|v_1-v_2\|_{\infty}
\end{equation*}
\end{lemma}
\begin{proof}

Fix $s$. The reward terms cancel in the difference, hence for any $(\pi,P)$,
\begin{align*}
|f_{\pi,P}(v_1)(s)-f_{\pi,P}(v_2)(s)|
&=\gamma\Big|\sum_{a\in\cA}\pi(a|s)\sum_{s'\in\cS}P_{s,a}(s')\big(v_1(s')-v_2(s')\big)\Big|\\
&\le \gamma \sum_{a\in\cA}\pi(a|s)\sum_{s'\in\cS}P_{s,a}(s')\,\|v_1-v_2\|_\infty \\
&=\gamma\,\|v_1-v_2\|_\infty,
\end{align*}
where we used nonnegativity and normalization of $\pi(\cdot|s)$ and $P_{s,a}(\cdot)$ so the weights sum to $1$. Taking $\sup_{\pi\in\Pi}\sup_{P\in\cP}$ does not increase the right-hand side, yielding the claim.
\end{proof}

\begin{lemma}
  $\mathcal{T}^*$ and $\hat{\mathbf{T}}^*$ are $\gamma$-contraction operators on $(\cS\ra \R,\norminf{\cd})$; i.e. for all $v_1,v_2:\cS\ra\R,$ 
  \begin{gather*}
  \|\mathcal{T}^*(v_1)-\mathcal{T}^*(v_2)\|_{\infty}\leq \gamma \|v_1-v_2\|_{\infty}, \\
  \|\hat{\mathbf{T}}^*(v_1)-\hat{\mathbf{T}}^*(v_2)\|_{\infty}\leq \gamma\|v_1-v_2\|_{\infty}.
  \end{gather*}

  \label{lemma:gamma-contraction-operator-1}
\end{lemma}
\begin{proof}
  By definition, we have
  \begin{align*}
  |\mathcal{T}^*(v_1)(s)-\mathcal{T}^*(v_2)(s)|&=\left|\sup_{\pi\in\Pi} \mathcal{T}^\pi(v_1)(s)-\sup_{\pi\in\Pi} \mathcal{T}^\pi(v_2)(s)\right|\\
  &=\left|\sup_{\pi\in\Pi}\inf_{P\in\mathcal{P}} f_{\pi,P}(v_1)(s)-\sup_{\pi\in\Pi}\inf_{P\in\mathcal{P}} f_{\pi,P}(v_2)(s)\right|.
  \end{align*}
  Since $|\sup_X f - \sup_X g|\leq \sup_X |f-g|$ and $|\inf_X f - \inf_X g|\leq \sup_X |f-g|$, we have
  \begin{equation*}
  |\mathcal{T}^*(v_1)(s)-\mathcal{T}^*(v_2)(s)|\leq\sup_{\pi\in\Pi}\sup_{P\in\mathcal{P}}\left| f_{\pi,P}(v_1)(s)-f_{\pi,P}(v_2)(s)\right|\leq\gamma \|v_1-v_2\|_{\infty}, 
  \end{equation*}
where the last inequality follows from Lemma~\ref{lemma:Lipschitz-continuity}. 
The above inequality holds for any $s\in\mathcal{S}$, which lead to 
  \begin{equation*}
  \|\mathcal{T}^*(v_1)-\mathcal{T}^*(v_2)\|_{\infty}\leq \gamma\|v_1-v_2\|_{\infty}.
  \end{equation*}
  We replace $P_s$ with $P_{s,n}$, therefore
  \begin{align*}
  |\hat{\mathbf{T}}^*(v_1)(s)-\hat{\mathbf{T}}^*(v_2)(s)|&\leq\sup_{\pi\in\Pi}\sup_{P\in\mathcal{P}_n}\left| f_{\pi,P}(v_1)(s)-f_{\pi,P}(v_2)(s)\right|\leq\gamma \|v_1-v_2\|_{\infty}.
  \end{align*}
  which lead to 
  \begin{equation*}
  \|\hat{\mathbf{T}}^*(v_1)-\hat{\mathbf{T}}^*(v_2)\|_{\infty}\leq \gamma\|v_1-v_2\|_{\infty}.
  \end{equation*}
\end{proof}

\begin{lemma}
  $\mathcal{T}^\pi$ and $\hat{\mathbf{T}}^\pi$ are $\gamma$-contraction operators on $(\cS\ra \R,\norminf{\cd})$; i.e. for all $v_1,v_2:\cS\ra\R,$ 
  \begin{gather*}
  \|\mathcal{T}^\pi(v_1)-\mathcal{T}^\pi(v_2)\|_{\infty}\leq \gamma \|v_1-v_2\|_{\infty}, \\
  \|\hat{\mathbf{T}}^\pi(v_1)-\hat{\mathbf{T}}^\pi(v_2)\|_{\infty}\leq \gamma\|v_1-v_2\|_{\infty}.
  \end{gather*}

  \label{lemma:gamma-contraction-operator-2}
\end{lemma}
\begin{proof}
  By definition, we have
  \begin{equation*}
  |\mathcal{T}^\pi(v_1)(s)-\mathcal{T}^\pi(v_2)(s)|=\left|\inf_{P\in\mathcal{P}} f_{\pi,P}(v_1)(s)-\inf_{P\in\mathcal{P}} f_{\pi,P}(v_2)(s)\right|
  \end{equation*}
  Since $|\inf_X f - \inf_X g|\leq \sup_X |f-g|$, we have
  \begin{align*}
  |\mathcal{T}^\pi(v_1)(s)-\mathcal{T}^\pi(v_2)(s)|\leq\sup_{P\in\mathcal{P}}\left| f_{\pi,P}(v_1)(s)-f_{\pi,P}(v_2)(s)\right|\leq\sup_{\tilde{\pi}\in\Pi}\sup_{P\in\mathcal{P}}\left| f_{\tilde{\pi},P}(v_1)(s)-f_{\tilde{\pi},P}(v_2)(s)\right|\leq\gamma \|v_1-v_2\|_{\infty} 
  \end{align*}
where the last inequality follows from Lemma~\ref{lemma:Lipschitz-continuity}. 
The above inequality holds for any $s\in\mathcal{S}$, which lead to 
  \begin{equation*}
  \|\mathcal{T}^\pi(v_1)-\mathcal{T}^\pi(v_2)\|_{\infty}\leq \gamma\|v_1-v_2\|_{\infty}.
  \end{equation*}
  We replace $P_s$ with $P_{s,n}$, therefore,
  \begin{align*}
  |\hat{\mathbf{T}}^\pi(v_1)(s)-\hat{\mathbf{T}}^\pi(v_2)(s)|&\leq \sup_{\pi\in\Pi}\sup_{P\in\mathcal{P}_n}\left| f_{\pi,P}(v_1)(s)-f_{\pi,P}(v_2)(s)\right|\leq\gamma \|v_1-v_2\|_{\infty}.
  \end{align*}
  which lead to 
  \begin{equation*}
  \|\hat{\mathbf{T}}^\pi(v_1)-\hat{\mathbf{T}}^\pi(v_2)\|_{\infty}\leq \gamma\|v_1-v_2\|_{\infty}.
  \end{equation*}
\end{proof}

\subsection{Proof of Proposition~\ref{proposition:value-function-error}}

\begin{proof}
  The proof of Proposition~\ref{proposition:value-function-error} follows a similar argument to that used for the continuous-case operator in \citet{wang2024statistical}.

  Let $u_0\equiv 0$ and $u_{k+1} = \hat{T}(u_k)$. $\hat{u}$ and $u^*$ is defined as the fixed point of $\hat{T}$ and $T$ respectively.
  \begin{align*}
  \Delta_{k+1} &= u_{k+1} - u^*\\
  &=\hat{T}(u_k)-\hat{T}(u^*)+\hat{T}(u^*)-T(u^*)\\
  &=\left[\hat{T}(u^*+\Delta_k)-\hat{T}(u^*)\right]+\left[\hat{T}(u^*)-T(u^*)\right]\\
  &:=\mathbf{H}(\Delta_k)+V
  \end{align*}
  Since $\hat{T}$ is a $\gamma$-contraction operator, we have
  \begin{align*}
  \left\|\mathbf{H}(\Delta_1)-\mathbf{H}(\Delta_2)\right\|_{\infty} = \left\|\hat{T}(u^*+\Delta_1)-\hat{T}(u^*+\Delta_2)\right\|_{\infty}\leq \gamma\|\Delta_1-\Delta_2\|_{\infty},
  \end{align*}
  therefore, $\mathbf{H}$ is also a $\gamma$-contraction operator. Then we show
  \begin{equation*}
  \|\Delta_k\|_{\infty}\leq \gamma^k\|u^*\|_{\infty}+\sum_{j=0}^{k-1}\gamma^j\|V\|_{\infty}
  \end{equation*}
  by induction: for $k=1$,
  \begin{align*}
  \|\Delta_1\|_{\infty}&\leq\|\mathbf{H}(\Delta_0)\|_{\infty}+\|V\|_{\infty}\\
  &=\|\mathbf{H}(\Delta_0)-\mathbf{H}(0)\|_{\infty}+\|V\|_{\infty}\\
  &\leq\gamma\left\|u^{*}\right\|_{\infty}+\left\|V\right\|_{\infty}.
  \end{align*}
  For any $k$, we have 
  \begin{align*}
  \|\Delta_{k+1}\|_{\infty}&\leq\|\mathbf{H}(\Delta_k)\|_{\infty}+\|V\|_{\infty}\\
  &=\|\mathbf{H}(\Delta_k)-\mathbf{H}(0)\|_{\infty}+\|V\|_{\infty}\\
  &\leq\gamma\left\|\Delta_k\right\|_{\infty}+\left\|V\right\|_{\infty}\\
  &\leq\gamma\left(\gamma^{k}\|u^*\|_{\infty}+\sum_{j=0}^{k-1}\gamma^j\|V\|_{\infty}\right)+\left\|V\right\|_{\infty}\\
  &=\gamma^{k+1}\|u^*\|_{\infty}+\sum_{j=0}^{k}\gamma^j\|V\|_{\infty}.
  \end{align*}
  Therefore,
  \begin{equation*}
  \|\hat{u}-u^{*}\|_{\infty}=\lim_{k\to\infty}\|\Delta_k\|_{\infty}\leq\sum_{j=0}^\infty\gamma^j\left\|V\right\|_{\infty}=\frac{1}{1-\gamma}\left\|\hat{T}(u^{*})-T(u^{*})\right\|_\infty.
  \end{equation*}

\end{proof}

\subsection{Proof of Corollary \ref{corollary:value-function-error-1}}
\begin{proof}
  By Lemma~\ref{lemma:gamma-contraction-operator-1}, both $\cT^*$ and $\hat{\mathbf T}^*$ are $\gamma$-contractions. Applying Proposition~\ref{proposition:value-function-error} with $T=\cT^*$, $\hat T=\hat{\mathbf T}^*$ and $u^*=v^*$ yields
\[
\|\hat v-v^*\|_\infty
\;\le\;\frac{1}{1-\gamma}\,\|\hat{\mathbf T}^*(v^*)-\cT^*(v^*)\|_\infty,
\]
which proves the statement.
\end{proof}

\subsection{Proof of Corollary \ref{corollary:value-function-error-2}}
\begin{proof}
  Fix any $\pi\in\Pi$. By Lemma~\ref{lemma:gamma-contraction-operator-2}, $\cT^\pi$ and $\hat{\mathbf T}^\pi$ are $\gamma$-contractions. Applying Proposition~\ref{proposition:value-function-error} with $T=\cT^\pi$, $\hat T=\hat{\mathbf T}^\pi$ and $u^*=V_{\cP}^{\pi}$ gives
\[
\|V_{\hat\cP}^{\pi}-V_{\cP}^{\pi}\|_\infty
\;\le\;\frac{1}{1-\gamma}\,\big\|\hat{\mathbf T}^\pi\!\left(V_{\cP}^{\pi}\right)-\cT^\pi\!\left(V_{\cP}^{\pi}\right)\big\|_\infty.
\]
Since the above inequality holds for any $\pi\in\Pi$, taking the supremum over $\pi\in\Pi$ on both sides yields
\[
\sup_{\pi\in\Pi}\|V_{\hat\cP}^{\pi}-V_{\cP}^{\pi}\|_\infty
\;\le\;\frac{1}{1-\gamma}\,\sup_{\pi\in\Pi}\big\|\hat{\mathbf T}^\pi\!\left(V_{\cP}^{\pi}\right)-\cT^\pi\!\left(V_{\cP}^{\pi}\right)\big\|_\infty,
\]
which proves the claim.
\end{proof}

\section{Strong Duality for Divergence-Based S-Rectangular Bellman Operators}
The proofs for all $f$-divergence-based uncertainty sets follow a unified framework. We first present Lemma~\ref{lemma:f-divergence-dual-problem}, which gives a general dual formulation for any convex $f$-divergence. For the KL-divergence and the $f_k$-divergence, we specialise this result by substituting the corresponding conjugate functions $f^*$. The detailed derivations for the KL-divergence and the $f_k$-divergence are provided in Appendix~\ref{sub:kl-dual-proof} and~\ref{sub:f-dual-proof}, respectively.
\begin{lemma}
  For any $f$-divergence uncertainty set, where $f:\mathbb{R}_{+}\rightarrow \mathbb{R}$ is a convex function and $f(1)=0$ and satisfies $f(0) = \lim_{t\da 0}f(t)$, the convex optimization problem
  \begin{align*}
  \inf_{P\in\mathcal{P}}&~\sum_{a\in\mathcal{A}}\pi(a|s)\mathbb{E}_{P_{s,a}}\left[R(s,a,S)+\gamma v(S)\right]
  \end{align*}
  can be reformulated as:
  \begin{equation*}
  \sup_{\lambda\geq 0,\boldsymbol{\eta}\in\mathbb{R}^{|\mathcal{A}|}} -\lambda\sum_{a\in\mathcal{A}}\mathbb{E}_{\overline{P}_{s,a}}\left[f^*\left(\frac{\eta_a-\pi(a|s)\left(R(s,a,S)+\gamma v(S)\right)}{\lambda}\right)\right]-\lambda |\mathcal{A}|\rho+\sum_{a\in\mathcal{A}}\eta_a
  \end{equation*}
  where $f^*(t) = -\inf_{s\geq 0}\left(f(s)-st\right)$.
  \label{lemma:f-divergence-dual-problem}
\end{lemma}
\begin{proof}
  We follow the proof of Lemma 8.5 in \cite{yang2022toward}, however, in our case, $R$ is determined by the next state.
  We do a change of variables, let $L_{s,a}(s^\prime) = \frac{P_{s,a}(s^\prime)}{\overline P_{s,a}(s^\prime)}$. The original optimization problem can be reformulated as:
  \begin{align*}
  \inf_{L_s\geq 0}&~\sum_{a\in\mathcal{A}}\pi(a|s)\mathbb{E}_{\overline{P}_{s,a}}\left[L_{s,a}\left(R(s,a,S)+\gamma v(S)\right)\right]\\
  {\rm s.t.}&~\sum_{a\in\mathcal{A}}\mathbb E_{\overline{P}_{s,a}}[f(L_{s,a})]\leq |\mathcal{A}|\rho\\
  &~\mathbb E_{\overline{P}_{s,a}}[L_{s,a}]=1\quad \text{for all }a\in\mathcal A
  \end{align*}
  The Lagrange function of primal problem is 
  \begin{align*}
  \mathcal{L}(L,\lambda,\boldsymbol{\eta})&=\sum_{a\in\mathcal{A}}\pi(a|s)\mathbb{E}_{\overline{P}_{s,a}}\left[L_{s,a}\left(R(s,a,S)+\gamma v(S)\right)\right] \\
  &+ \lambda\left(\sum_{a\in\mathcal{A}}\mathbb E_{\overline{P}_{s,a}}[f(L_{s,a})]- |\mathcal{A}|\rho\right)-\sum_{a\in\mathcal{A}}\eta_a\left(\mathbb E_{\overline{P}_{s,a}}[L_{s,a}]-1\right)
  \end{align*}
  Denoting $f^*(t) = -\inf_{s\geq 0}\left(f(s)-st\right)$,
  \begin{align*}
  &\inf_{L_s\geq 0}~\mathcal{L}(L,\lambda,\boldsymbol{\eta})\\
  &=\inf_{L_s\geq 0}\left(\sum_{a\in\mathcal{A}}\mathbb{E}_{\overline{P}_{s,a}}\Big[\pi(a|s)L_{s,a}\left(R(s,a,S)+\gamma v(S)\right)+\lambda f(L_{s,a})-\eta_a L_{s,a}\Big]\right)-\lambda |\mathcal{A}|\rho+\sum_{a\in\mathcal{A}}\eta_a\\
  &=\lambda\sum_{a\in\mathcal{A}}\inf_{L_{s,a}\geq 0}\mathbb{E}_{\overline{P}_{s,a}}\left[\frac{\pi(a|s)\left(R(s,a,S)+\gamma v(S)\right)-\eta_a}{\lambda}L_{s,a}+ f(L_{s,a})\right]-\lambda |\mathcal{A}|\rho+\sum_{a\in\mathcal{A}}\eta_a\\
  &=-\lambda\sum_{a\in\mathcal{A}}\mathbb{E}_{\overline{P}_{s,a}}\left[f^*\left(\frac{\eta_a-\pi(a|s)\left(R(s,a,S)+\gamma v(S)\right)}{\lambda}\right)\right]-\lambda |\mathcal{A}|\rho+\sum_{a\in\mathcal{A}}\eta_a
  \end{align*}
\end{proof}

\subsection{Proof of Lemma~\ref{lemma:kl-divergence-bellman-operator}}
\label{sub:kl-dual-proof}
\begin{proof}
  Recall that for the KL-divergence, $f(t) = t\log t$, whose conjugate function $f^*(s) = e^{s-1}$. Substituting $f^*$ into Lemma~\ref{lemma:f-divergence-dual-problem}, we obtain the following dual form:
  \begin{align}
  &\sup_{\lambda\geq 0,\boldsymbol{\eta}\in\mathbb{R}^{|\mathcal{A}|}} -\lambda\sum_{a\in\mathcal{A}}\mathbb{E}_{\overline{P}_{s,a}}\left[f^*\left(\frac{\eta_a-\pi(a|s)\left(R(s,a,S)+\gamma v(S)\right)}{\lambda}\right)\right]-\lambda |\mathcal{A}|\rho+\sum_{a\in\mathcal{A}}\eta_a\notag\\
  \begin{split}
    &=\sup_{\lambda\geq 0,\boldsymbol{\eta}\in\mathbb{R}^{|\mathcal{A}|}} -\lambda\sum_{a\in\mathcal{A}}\exp\left(\frac{\eta_a-\lambda}{\lambda}\right)\mathbb{E}_{\overline{P}_{s,a}}\left[\exp\left(\frac{-\pi(a|s)\left(R(s,a,S)+\gamma v(S)\right)}{\lambda}\right)\right]\\
    &\quad -\lambda |\mathcal{A}|\rho+\sum_{a\in\mathcal{A}}\eta_a.
  \end{split} \label{eq:kl-dual-with-eta}
  \end{align}

We first note that for each action $a$, the term $\lambda \mathbb{E}_{\overline{P}_{s,a}}[\cdot]$ is a positive constant with respect to $\eta_a$, while the term $-\lambda\mathbb{E}_{\overline{P}_{s,a}}[\cdot]\exp((\eta_a-\lambda)/\lambda)$ is concave in $\eta_a$, since for any $c>0$, the function $-c\exp(x)$ is concave. Moreover, the term $\sum_{a}\eta_a$ is affine, and hence concave. As the sum of concave functions is concave, we conclude that (\ref{eq:kl-dual-with-eta}) is concave in $\boldsymbol{\eta}$. Next, we optimize with respect to $\boldsymbol{\eta}$ by setting the gradient with respect to each $\eta_a$ to zero:
  \begin{equation*}
  -\exp\left(\frac{\eta_a-\lambda}{\lambda}\right)\mathbb{E}_{\overline{P}_{s,a}}\left[\exp\left(\frac{-\pi(a|s)\left(R(s,a,S)+\gamma v(S)\right)}{\lambda}\right)\right]+1=0
  \end{equation*}
  Solving for $\eta_a$, we obtain
  \begin{equation}
  \eta_a=\lambda-\lambda\log\mathbb{E}_{\overline{P}_{s,a}}\left[\exp\left(\frac{-\pi(a|s)\left(R(s,a,S)+\gamma v(S)\right)}{\lambda}\right)\right].\label{eq:kl-dual-eta-solution}
  \end{equation}
  Substituting (\ref{eq:kl-dual-eta-solution}) into (\ref{eq:kl-dual-with-eta}), we obtain
  \begin{equation*}
  \sup_{\lambda\geq 0,\boldsymbol{\eta}\in\mathbb{R}^{|\mathcal{A}|}} -\lambda\sum_{a\in\mathcal{A}}\log\mathbb{E}_{\overline{P}_{s,a}}\left[\exp\left(\frac{-\pi(a|s)\left(R(s,a,S)+\gamma v(S)\right)}{\lambda}\right)\right]-\lambda |\mathcal{A}|\rho.
  \end{equation*}
\end{proof}

\subsection{Proof of Lemma \ref{lemma:f-divergence-bellman-operator}}
\label{sub:f-dual-proof}
\begin{proof}
We first introduce the conjugate function of $f_k$, which will be instrumental for deriving the dual representation of DR Bellman operator.
\begin{lemma}[\citet{duchiLearningModelsUniform2021}, Section 2]
  Recall that in $f_k$-divergence,
  \begin{equation*}
  f_k(t):=\frac{t^k-kt+k-1}{k(k-1)}
  \end{equation*}
  The conjugate function $f_k^*(s)=\sup_{t\geq 0}\left(st-f_k(t)\right)$ is given by
  \begin{equation*}
  f_k^*(s):=\frac1k\left[\left((k-1)s+1\right)_+^{k^*}-1\right]
  \end{equation*}
where $(x)_+= \max(x, 0)$.
\label{lemma:conjugate-function-fk}
\end{lemma}

Substituting $f_k^*$ into Lemma~\ref{lemma:f-divergence-dual-problem}, and let $w_{s,a}(S) := \pi(a|s)\left(R(s,a,S)+\gamma v(S)\right)$, we obtain
\begin{align*}
&\sup_{\lambda\geq0,\eta\in\mathbb{R}^{|\mathcal{A}|}}-\sum_{a\in\mathcal{A}}\lambda\mathbb{E}_{\overline{P}_{s,a}}  \left[f^*\left(\frac{\eta_a-w_{s,a}(S)}{\lambda}\right)\right]-\lambda|\mathcal{A}|\rho+\sum_{a\in\mathcal{A}}\eta_a\\
&=\sup_{\lambda\geq 0,\eta\in\mathbb{R}^{|\mathcal{A}|}} -\sum_{a\in\mathcal{A}}\lambda \mathbb{E}_{\overline{P}_{s,a}}\left[\frac1k\left[\left((k-1)\frac{\eta_a-w_{s,a}(S)}{\lambda}+1\right)_+^{k^*}-1\right]\right]
\end{align*}

Since $k-1 > 0$ and $\lambda > 0$ are constants with respect to the random variable $S$, we can factor them out of the expectation and the positive-part operator $(\cdot)_+$.
\begin{align*}
&=\sup_{\lambda\geq0,\eta\in\mathbb{R}^{|\mathcal{A}|}}-\frac{(k-1)^{k^*}}{k\lambda^{k^*-1}}\sum_{a\in\mathcal{A}}\mathbb{E}_{\overline{P}_{s,a}}\left[\left(\eta_a-w_{s,a}(S)+\frac{\lambda}{k-1} \right)_{+}^{k^*} \right] -\lambda|\mathcal{A}|\left(\rho-\frac{1}k\right)+\sum_{a\in\mathcal{A}}\eta_a
\end{align*}

Finally, we perform the change of variables, let $\tilde{\eta}_a = \eta_a + \frac{\lambda}{k-1}$, we obtain
\begin{align}
&=\sup_{\lambda\geq0,\tilde\eta\in\mathbb{R}^{|\mathcal{A}|}}-\frac{(k-1)^{k^*}}{k\lambda^{k^*-1}}\sum_{a\in\mathcal{A}}\mathbb{E}_{\overline{P}_{s,a}}\left[\left(\tilde{\eta}_a-w_{s,a}(S)\right)_{+}^{k^*} \right] -\lambda|\mathcal{A}|\left(\rho+\frac{1}{k(k-1)}\right)+\sum_{a\in\mathcal{A}} \tilde{\eta}_a
\label{eq: proof-of-lemma-2-1}
\end{align}
Since $-\lambda^{-\alpha}$ is concave in $\lambda$ for any $\alpha > 0$, and $\lambda |\mathcal{A}| \left( \rho + \frac{1}{k(k-1)} \right)$ is an affine function of $\lambda$, it follows that \eqref{eq: proof-of-lemma-2-1} is concave with respect to $\lambda$. To optimize over $\lambda$, we take the derivative with respect to $\lambda$ and set it to zero, which yields:
\begin{equation*}
\frac{(k-1)^{k^*}}{k(k-1)\lambda^{k^*}}\sum_{a\in\mathcal{A}}\mathbb{E}_{\overline{P}_{s,a}}\left[\left(\tilde{\eta}_a-w_{s,a}(S)\right)_{+}^{k^*} \right] -|\mathcal{A}|\left(\rho+\frac{1}{k(k-1)}\right)=0
\end{equation*}

Multiply $k(k-1)$ on both side of the equation, we have
\begin{equation*}
\frac{(k-1)^{k^*}}{\lambda^{k^*}}\sum_{a\in\mathcal{A}}\mathbb{E}_{\overline{P}_{s,a}}\left[\left(\tilde{\eta}_a-w_{s,a}(S)\right)_{+}^{k^*} \right] -|\mathcal{A}|\left(k(k-1)\rho+1\right)=0
\end{equation*}
Therefore, we obtain
\begin{equation*}
\lambda^* = (k-1)|\mathcal{A}|^{-1/k^*}\left(k(k-1)\rho+1 \right)^{-1/k^*}\left(\sum_{a\in\mathcal{A}}\mathbb{E}_{\overline{P}_{s,a}}\left[(\tilde\eta_a-w_{s,a}(S))_+^{k^*} \right] \right)^{1/k^*}
\end{equation*}
By substituting $\lambda^*$ into the equation (\ref{eq: proof-of-lemma-2-1}) , we have
\begin{align*}
&\sup_{\lambda\geq0,\eta\in\mathbb{R}^{|\mathcal{A}|}}-\sum_{a\in\mathcal{A}}\lambda\mathbb{E}_{\overline{P}_{s,a}}  \left[f^*\left(\frac{\eta_a-w_{s,a}(S)}{\lambda}\right)\right]-\lambda|\mathcal{A}|\rho+\sum_{a\in\mathcal{A}}\tilde{\eta}_a\\
&=\sup_{\tilde\eta\in\mathbb{R}^{|\mathcal{A}|}}-\frac{k-1}k|\mathcal{A}|^{1/k}\left(k(k-1)\rho+1 \right)^{1/k}\left(\sum_{a\in\mathcal{A}}\mathbb{E}_{\overline{P}_{s,a}}\left[(\tilde\eta_a-w_{s,a}(S))_+^{k^*} \right] \right)^{1/k^*}\\
&-\frac1k|\mathcal{A}|^{1/k}\left(k(k-1)\rho+1 \right)^{1/k}\left(\sum_{a\in\mathcal{A}}\mathbb{E}_{\overline{P}_{s,a}}\left[(\tilde\eta_a-w_{s,a}(S))_+^{k^*} \right] \right)^{1/k^*}+\sum_{a\in\mathcal{A}} \tilde{\eta}_a \\
&=\sup_{\tilde\eta\in\mathbb{R}^{|\mathcal{A}|}}-|\mathcal{A}|^{1/k}\left(k(k-1)\rho+1 \right)^{1/k}\left(\sum_{a\in\mathcal{A}}\mathbb{E}_{\overline{P}_{s,a}}\left[(\tilde\eta_a-w_{s,a}(S))_+^{k^*} \right]\right)^{1/k^*}+\sum_{a\in\mathcal{A}} \tilde{\eta}_a
\end{align*}
\end{proof}

\section{Proofs of Properties of the Empirical Bellman Operator: KL Case}
\label{sec:proofs_kl_case}

Our techniques in this section refine that in \citet{wang_sample_2024}. To follow the constructions in \citet{wang_sample_2024}, we introduce some notations. Consider $\mu_{s,a}\in\Delta(\cS)$ and its empirical version $\mu_{s,a,n}$ constructed from $n$ i.i.d samples from $\mu_{s,a}$. Define the collection of these measures under state $s$ as $\boldsymbol{\mu}_s:=\{\mu_{s,a}:a\in\mathcal{A}\}$. 
For a function $u: \mathcal{S} \to \mathbb{R}$ and for each $s \in \mathcal{S}$, we define:
\begin{align*}
\|u\|_{\infty, \boldsymbol{\mu}_s} &= \max_{a\in\mathcal A}\|u\|_{L^{\infty}(\mu_{s,a})}, \\
\left\|\frac{dm_{n}}{d\mu_{n}(t)}\right\|_{\infty, \boldsymbol{\mu}_s} &= \max_{a\in\mathcal A}\left\|\frac{dm_{a,n}}{d\mu_{a,n}(t)}\right\|_{L^{\infty}(\mu_{s,a})}.
\end{align*}
For the supremum over all states, we define
\begin{equation*}
\|u\|_\infty = \sup_{s \in \mathcal{S}} \|u\|_{\infty, \boldsymbol{\mu}_s}.
\end{equation*}

We define a "good event" under which the empirical measure $\mu_{s,a,n}$ uniformly approximates the population measure $\mu_{s,a}$ with relative error bounded by $\delta_0$ across all actions $a\in\mathcal{A}$. Formally, this event is given by
\begin{equation*}
\Omega_{n,\delta_0}(\boldsymbol\mu_s)=\left\{\omega:\sup_{a\in\mathcal{A}}\sup_{s^\prime:\mu_{s,a}(s^\prime)>0}\left|\frac{\mu_{s,a,n}(\omega)(s^\prime)-\mu_{s,a}(s^\prime)}{\mu_{s,a}(s^\prime)}\right|\leq \delta_0\right\}.
\end{equation*}
Further, the good event over all states is defined as
\begin{equation*}
\Omega_{n,\delta_0}=\bigcap_{s\in\mathcal{S}}\Omega_{n,\delta_0}(\boldsymbol{\mu}_s)=\left\{\omega:\sup_{s\in\mathcal{S}}\sup_{a\in\mathcal{A},s^\prime:\mu_{s,a}(s^\prime)>0}\left|\frac{\mu_{s,a,n}(\omega)(s^\prime)-\mu_{s,a}(s^\prime)}{\mu_{s,a}(s^\prime)}\right|\leq \delta_0\right\}.
\end{equation*}

For notation simplicity, we suppress the dependence on the state variable $s$. 
Consider a function $u : \mathcal{S} \to \mathbb{R}$. The dual function under KL-divergence is given by:
\begin{equation}
f(\boldsymbol{\mu},u,\lambda):=-\lambda|\mathcal{A}|\rho-\sum_{a\in\mathcal A}\lambda\log \mu_{a}\left[e^{-d_au/\lambda}\right],
\label{eq:f-kl}
\end{equation}
where $\lambda >0$ is the dual regularization parameter, and we denote $d_a:=\pi(a|s)$ for simplicity.

We define the deviation between empirical and true measures as
\begin{equation*}
  m_{a,n} = \mu_{a,n} - \mu_a,
\end{equation*}
and their convex interpolation by
\begin{equation*}
  \mu_{a,n}(t) = t\mu_{a} + (1-t)\mu_{a,n}.
\end{equation*}


\subsection{Proof of Proposition \ref{proposition:kl-bellman-operator-error}}
\label{sub:proof_of_proposition_kl}

\begin{proof}
Using (\ref{eq:f-kl}) to express Bellman operator, we obtain
\begin{equation}\label{eqn:Thatv_minus_Tv}
\begin{aligned}
\sup_{\pi\in\Pi}\left|\hat{\mathbf{T}}^\pi(v)(s) - \mathcal{T}^\pi(v)(s)\right|&= \sup_{d\in\Delta(\mathcal{A})}\left|\sup_{\lambda>0}f(P_{s,n},R(s,\cdot,\cdot)+\gamma v,\lambda)-\sup_{\lambda>0}f(P_s,R(s,\cdot,\cdot)+\gamma v,\lambda)\right|\\
&\leq\sup_{\lambda >0,d\in\Delta(\mathcal{A})} \left|f(P_{s,n},R(s,\cdot,\cdot)+\gamma v,\lambda) - f(P_s,R(s,\cdot,\cdot)+\gamma v,\lambda)\right|.
\end{aligned}
\end{equation}
We analyze the sensitivity of the mapping $\boldsymbol{\mu}\ra f(\boldsymbol{\mu},u,\lambda)$. For any fixed $u,\boldsymbol{\mu}$ and $\boldsymbol{\mu}_n$, define 
\begin{equation*}
g_n(t,\lambda) = f\left(\boldsymbol{\mu}_n(t),u,\lambda\right).
\end{equation*}
According to mean value theorem, there exists $\tau \in (0,1)$ satisfies:
\begin{align*}
\left|f(\boldsymbol{\mu}_n,u,\lambda) - f(\boldsymbol{\mu},u,\lambda)\right|&=\left|g_n(0,\lambda)-g_n(1,\lambda)\right|\\
&=\left|\partial_tg_n(t,\lambda)\Big|_{t=\tau}\right|\\
&=\left|\sum_{a\in\mathcal A}\lambda \frac{m_{a,n}[e^{-d_au/\lambda}]}{\mu_{a,n}(\tau)[e^{-d_au/\lambda}]}\right|
\end{align*}
To bound the difference above, we invoke the following lemma.
\begin{lemma}
  For any fixed $u$ and $\pi$, $\mu_n\ll \mu$, we have that
  \begin{equation*}
  \sup_{\lambda\geq 0}\left|\sum_{a\in\mathcal A}\lambda \frac{m_{a,n}[e^{-d_au/\lambda}]}{\mu_{a,n}(t)[e^{-d_au/\lambda}]}\right|\leq 2 \|u\|_{\infty}\left\|\frac{dm_{n}}{d\mu_{n}(t)}\right\|_{\infty,\boldsymbol\mu}.
  \end{equation*}
  \label{lemma:kl-sup-f-diff}
\end{lemma}
The proof is deferred to Appendix~\ref{sub:kl-sup-f-diff-proof}. According to lemma \ref{lemma:kl-sup-f-diff}, we have
\begin{equation*}
\sup_{\lambda\geq 0}|f(\boldsymbol\mu_n,u,\lambda)-f(\boldsymbol\mu,u,\lambda)|\leq 2 \|u\|_{\infty}\left\|\frac{dm_{n}}{d\mu_{n}(t)}\right\|_{\infty,\boldsymbol\mu}.
\end{equation*}

We decomposed the probability using the event $\Omega_{n,\delta_0}(\boldsymbol{\mu})$ where the empirical estimates are close to the population measures:
\begin{align*}
&P\left(\sup_{\lambda\geq 0,d\in\Delta(\mathcal{A})}|f(\mu_n,u,\lambda)-f(\mu,u,\lambda)|>t\right)\\
&\leq P(\Omega_{n,\delta_0}(\boldsymbol\mu)^c)+P\left(2\left\|u\right\|_{\infty}\left\|\frac{dm_n}{d\mu_n(\tau)}\right\|_{\infty,\boldsymbol\mu}>t,\Omega_{n,\delta_0}(\boldsymbol\mu)\right)
\end{align*}
To control the denominator $\mu_{a,n}(\tau)(s')$ appearing in the bound, we use the following lemma, which asserts that under the good event, the empirical and population measures remain close for all $t \in [0,1]$:
\begin{lemma}
  For any $s^\prime$ with $\mu(s^\prime)>0$, the measure $\mu_n(t)(s^\prime)$ satisfies
  \begin{equation*}
  (1-\delta_0)\mu(s^\prime)\leq\mu_n(t)(s^\prime)\leq (1+\delta_0)\mu(s^\prime),\quad \forall t\in[0,1].
  \end{equation*}
  \label{lemma:mut-near-mu}
\end{lemma}
The proof is deferred to Appendix~\ref{sub:mut-near-mu-proof}. By using lemma \ref{lemma:mut-near-mu}, we have $\mu_{a,n}(\tau)(s^\prime)\geq (1-\delta_0)\mu_a(s^\prime)$, therefore,
\begin{align*}
&\leq P\left(\sup_{a,s^\prime}\left|\frac{\mu_{a,n}(s^\prime)-\mu_a(s^\prime)}{\mu_a(s^\prime)}\right|>\delta_0\right) + P\left(2\|u\|_{\infty}\sup_{a,s^\prime}\left|\frac{\mu_{a,n}(s^\prime)-\mu_a(s^\prime)}{(1-\delta_0) \mu_{a}(s^\prime)}\right|>t\right).
\end{align*}

By using the multiplicative Chernoff bound and Bernstein inequality, we have
\begin{align*}
&\leq P\left(\sup_{a,s^\prime}\left|\frac 1n \sum_{i=1}^n \mathds{1}(S_i=s^\prime)-\mu_a(s^\prime)\right|>\delta_0\mu_a(s^\prime)\right) \\ &+ P\left(\frac{2}{1-\delta_0}\|u\|_{\infty}\sup_{a,s^\prime}\frac 1{\mu_{a}(s^\prime)}\left|\frac 1n \sum_{i=1}^n \mathds{1}(S_i=s^\prime)-\mu_a(s^\prime)\right|>t\right) \\
&\leq 2\sum_{a\in\mathcal{A}}\sum_{s^\prime\in \mathcal{S}}\left(\exp\left(-\frac{\delta_0^2n\mu_a(s^\prime)}{3}\right)+\exp\left(-\frac{t^2}2\left(\frac{4\|u\|_{\infty}^2}{(1-\delta_0)^2n\mu_a(s^\prime)} + \frac{2\|u\|_{\infty}t}{3(1-\delta_0)n\mu_a(s^\prime)}\right)^{-1}\right)\right).
\end{align*}
Since $\mu_a(s^\prime)\geq \mathfrak{p}_{\wedge}$, and both exponential term above is monotonically decreasing over $\mu_a(s^\prime)$, we have
\begin{align*}
&\leq 2|\mathcal{A}||\mathcal{S}|\left(\exp\left(-\frac{\delta_0^2n\mathfrak{p}_{\wedge}}{3}\right)+\exp\left(-\frac{t^2}2\left(\frac{4\|u\|_{\infty}^2}{(1-\delta_0)^2n\mathfrak{p}_{\wedge}} + \frac{2\|u\|_{\infty}t}{3(1-\delta_0)n\mathfrak{p}_{\wedge}}\right)^{-1}\right)\right).
\end{align*}
Recall from \eqref{eqn:Thatv_minus_Tv} that
\begin{equation*}
P\left(\sup_{\pi\in\Pi}\left|\hat{\mathbf{T}}^\pi(v)(s) - \mathcal{T}^\pi(v)(s)\right|>t\right) \leq P\left(\sup_{\lambda >0,d_a\in\Delta(\mathcal{A})} \left|f(P_{s},R(s,\cdot,\cdot)+\gamma v,\lambda) - f(P_s,R(s,\cdot,\cdot)+\gamma v,\lambda)\right|\right)
\end{equation*}
Replacing $\boldsymbol\mu$ with $P_s$ and $\boldsymbol\mu_n$ with $P_{s,n}$ and choose $\delta_0 = \frac12$, by union bound, we have
\begin{align*}
&P\left(\sup_{\pi\in\Pi}\left\|\hat{\mathbf{T}}^\pi(v)-\mathcal{T}^\pi(v)\right\|_\infty>t\right)\\
&\leq P\left(\sup_s\sup_{\lambda\geq0,d\in\Delta(\mathcal{A})}|f(P_{s,n},R(s,\cdot,\cdot)+\gamma v,\lambda)-f(P_{s},R(s,\cdot,\cdot)+\gamma v,\lambda)|>t\right)\\
&\leq 2|\mathcal{S}|^2|\mathcal{A}|\exp\left(-\frac{n\mathfrak{p}_{\wedge}}{12}\right)+2|\mathcal{S}|^2|\mathcal{A}|\exp\left(-\frac{t^2}{2}\left(\frac{16\|R(s,\cdot,\cdot)+\gamma v\|_{\infty}^2}{n\mathfrak{p}_{\wedge}} + \frac{4\|R(s,\cdot,\cdot)+\gamma v\|_{\infty}t}{3\gamma n\mathfrak{p}_{\wedge}}\right)^{-1}\right).
\end{align*}
Set each term to be less than $\eta/2$, we need
\begin{gather}
n\geq \frac{12}{\mathfrak{p}_{\wedge}}\log\left(4|\mathcal{S}|^2|\mathcal{A}|/\eta\right)\label{eq:proposition-kl-n}\\
t\geq \frac{8\|R+\gamma v\|_{\infty}}{3n\mathfrak{p}_{\wedge}}\log\left(4|\mathcal{S}|^2|\mathcal{A}|/\eta\right) + \frac{4 \|R+\gamma v\|_{\infty}}{\sqrt{n\mathfrak{p}_{\wedge}}}\sqrt{2\log\left(4|\mathcal{S}|^2|\mathcal{A}|/\eta\right)}.\label{eq:proposition-kl-t}
\end{gather}
Under \eqref{eq:proposition-kl-n}, we have
\begin{equation*}
\frac{\log(4|\cS|^2|\cA|/\eta)}{n\mathfrak{p}_\wedge}\leq \sqrt{\frac{\log(4|\cS|^2|\cA|/\eta)}{n\mathfrak{p}_\wedge}}.
\end{equation*}
By substituting this bound into \eqref{eq:proposition-kl-t}, we have
\begin{align*}
&\frac{8\|R+\gamma v\|_{\infty}}{3n\mathfrak{p}_{\wedge}}\log\left(4|\mathcal{S}|^2|\mathcal{A}|/\eta\right) + \frac{4 \|R+\gamma v\|_{\infty}}{\sqrt{n\mathfrak{p}_{\wedge}}}\sqrt{2\log\left(4|\mathcal{S}|^2|\mathcal{A}|/\eta\right)}\\
&\leq \left(\frac{8}{3}+4\sqrt{2}\right)\frac{\|R+\gamma v\|_{\infty}}{\sqrt{n\mathfrak{p}_{\wedge}}}\sqrt{\log\left(4|\mathcal{S}|^2|\mathcal{A}|/\eta\right)}\\
&\leq \frac{9\|R+\gamma v\|_{\infty}}{\sqrt{n\mathfrak{p}_{\wedge}}}\sqrt{\log\left(4|\mathcal{S}|^2|\mathcal{A}|/\eta\right)}
\end{align*}
Therefore, for when $n$ specifies \eqref{eq:proposition-kl-n} and $t$ satisfies
\begin{equation*}
t \geq \frac{9\|R+\gamma v\|_{\infty}}{\sqrt{n\mathfrak{p}_{\wedge}}}\sqrt{\log\left(4|\mathcal{S}|^2|\mathcal{A}|/\eta\right)},
\end{equation*}
we have
\begin{equation*}
P\left(\sup_{\pi\in\Pi}\left\|\hat{\mathbf{T}}^\pi(v)-\mathcal{T}^\pi(v)\right\|_\infty>t\right) \leq \eta.
\end{equation*}
Since $|\sup_X f-\sup_X g|\leq \sup_X|f-g|$,
\begin{align*}
P\left(\sup_{\pi\in\Pi}\left\|\hat{\mathbf{T}}^*(v)-\mathcal{T}^*(v)\right\|_\infty>t\right)&=P\left(\left\|\sup_{\pi\in\Pi}\hat{\mathbf{T}}^\pi(v)-\sup_{\pi\in\Pi}\mathcal{T}^\pi(v)\right\|_\infty>t\right)\\
&\leq P\left(\sup_{\pi\in\Pi}\left\|\hat{\mathbf{T}}^\pi(v)-\mathcal{T}^\pi(v)\right\|_\infty>t\right) \\
&\leq \eta.
\end{align*}

This implies Proposition \ref{proposition:kl-bellman-operator-error}.
\end{proof}


\subsection{Proof of Lemma \ref{lemma:kl-sup-f-diff}}
\label{sub:kl-sup-f-diff-proof}
\begin{proof}

Observe that multiplying the numerator and denominator by $e^{d_a\|u\|_{L^\infty(\mu_a)}/\lambda}$ preserves the value of the fraction. This is equivalent to rewriting the exponential terms as:

\begin{equation*}
\left|\sum_{a\in\mathcal A}\lambda \frac{m_{a,n}[e^{-d_au/\lambda}]}{\mu_{a,n}(t)[e^{-d_au/\lambda}]}\right| =\left|\sum_{a\in\mathcal A}\lambda \frac{m_{a,n}[e^{d_a(\|u\|_{L^\infty(\mu_a)} -u)/\lambda}]}{\mu_{a,n}(t)[e^{d_a(\|u\|_{L^\infty(\mu_a)} -u)/\lambda}]}\right|.
\end{equation*}

Since $m_{a,n} = \mu_{a,n}-\mu_a$, for any constant $c$, we have $m_{a,n}[c]=0$, which lead to
\begin{equation*}
= \left|\sum_{a\in\mathcal A}\lambda \frac{m_{a,n}[e^{d_a(\|u\|_{L^\infty(\mu_a)} -u)/\lambda}-1]}{\mu_{a,n}(t)[e^{d_a(\|u\|_{L^\infty(\mu_a)} -u)/\lambda}]}\right|.
\end{equation*}
For any measure $m,\mu$ and random variable $w_1,w_2$, the following equation holds:
\begin{align*}
\left|\frac{m[w_1]}{\mu[w_2]}\right| &=\left|\frac{\sum_{s}m(s)w_1(s)}{\sum_{s}\mu(s)w_2(s)}\right|\\
&= \left|\left(\sum_{s}\mu(s)\frac{m(s)}{\mu(s)}w_2(s)\frac{w_1(s)}{w_2(s)}\right)\left(\sum_{s}\mu(s)w_2(s)\right)^{-1}\right|\\
&\leq\left|\frac{\sum_{s}\mu(s)w_2(s)}{\sum_{s}\mu(s)w_2(s)}\right|\cdot\max_s\left|\frac{m(s)}{\mu(s)}\right|\cdot\max_s\left|\frac{w_1(s)}{w_2(s)}\right|\\
&=\left\|\frac{dm}{d\mu}\right\|_{L^{\infty}(\mu)}\left\|\frac{w_1}{w_2}\right\|_{L^{\infty}(\mu)}.
\end{align*}
Applying this result and $\left|\sum \cdot\right|\leq \sum|\cdot|$, we obtain
\begin{equation*}
\left|\sum_{a\in\mathcal A}\lambda \frac{m_{a,n}[e^{-d_au/\lambda}]}{\mu_{a,n}(t)[e^{-d_au/\lambda}]}\right|\leq \sum_{a\in\mathcal A}\left\|\lambda \frac{e^{d_a(\|u\|_{L^\infty(\mu_a)} -u)/\lambda}-1}{e^{d_a(\|u\|_{L^\infty(\mu_a)} -u)/\lambda}}\right\|_{L^\infty(\mu_a)}\left\|\frac{dm_{a,n}}{d\mu_{a,n}(t)}\right\|_{L^\infty(\mu_a)}.
\end{equation*}
Notice that when $x>0$, we have $e^x-1 \leq xe^x$, then we obtain
\begin{align*}
&\leq\sum_{a\in\mathcal A}\left\|\lambda \frac{\frac{d_a(\|u\|_{L^\infty(\mu_a)} -u)}{\lambda}e^{d_a(\|u\|_{L^\infty(\mu_a)} -u)/\lambda}}{e^{d_a(\|u\|_{L^\infty(\mu_a)} -u)/\lambda}}\right\|_{L^\infty(\mu_a)}\left\|\frac{dm_{a,n}}{d\mu_{a,n}(t)}\right\|_{L^\infty(\mu_a)}\\
&\leq\sum_{a\in\mathcal A} \left\| d_a(\|u\|_{L^\infty(\mu_a)}-u)\right\|_{L^\infty(\mu_a)}\left\|\frac{dm_{a,n}}{d\mu_{a,n}(t)}\right\|_{L^\infty(\mu_a)}\\
&\leq\sum_{a\in\mathcal A} 2 d_a\|u\|_{\infty}\left\|\frac{dm_{a,n}}{d\mu_{a,n}(t)}\right\|_{L^\infty(\mu_a)}\\
&\leq 2\|u\|_{\infty}\left\|\frac{dm_{n}}{d\mu_{n}(t)}\right\|_{\infty,\boldsymbol\mu}
\end{align*}  
as claimed. 
\end{proof}


\subsection{Proof of Lemma \ref{lemma:mut-near-mu}}
\label{sub:mut-near-mu-proof}
\begin{proof}
On the event $\Omega_{n,\delta_0}$, the empirical measure satisfies $\sup_{s^\prime\in\mathcal{S}}\left|\frac{\mu_n(s^\prime)-\mu(s^\prime)}{\mu(s^\prime)}\right|\leq \delta_0$. Hence, for any $s^\prime$ with $\mu(s^\prime) > 0$, we have:
\begin{equation*}
(1-\delta_0)\mu(s^\prime)\leq \mu_n(s^\prime)\leq (1+\delta_0)\mu(s^\prime).
\end{equation*}
Substituting in the above bound on $\mu_n(s^\prime)$ into the definition of $\mu_{n}(t)(s^\prime)$ gives
\begin{equation*}
(1-(1-t)\delta_0)\mu(s^\prime)\leq t\mu(s^\prime)+(1-t)\mu_n(s^\prime)\leq(1+(1-t)\delta_0)\mu(s^\prime).
\end{equation*}
For all $t\in[0,1]$, $(1-t)\leq1$, therefore, we have 
\begin{equation*}
(1-\delta_0)\mu(s^\prime)\leq\mu_n(t)(s^\prime)\leq(1+\delta_0)\mu(s^\prime).
\end{equation*}
\end{proof}

\subsection{Proof of Theorem \ref{theorem: kl}}
\label{sub:proof_of_theorem_kl}
\begin{proof}
Substituting $\|R + \gamma v\|_{\infty} \leq 1/(1-\gamma)$ into the bound from Proposition~\ref{proposition:kl-bellman-operator-error} and applying Proposition~\ref{proposition:value-function-error}, we obtain the stated result.
\begin{align*}
\|\hat v - v^*\|_{\infty}&\leq \frac{1}{1-\gamma}\left\|\hat{\mathbf{T}}^*(v^*)-\mathcal{T}^*(v^*)\right\|_\infty\\
&\leq \frac{9\|R+\gamma v\|_{\infty}}{(1-\gamma)\sqrt{n\mathfrak{p}_{\wedge}}}\sqrt{\log\left(4|\mathcal{S}|^2|\mathcal{A}|/\eta\right)}\\
&\leq \frac{9}{(1-\gamma)^2\sqrt{n\mathfrak{p}_{\wedge}}}\sqrt{\log\left(4|\mathcal{S}|^2|\mathcal{A}|/\eta\right)}\\
\end{align*}
with probability $1-\eta$.

The preceding step provides a uniform bound on the value-function estimation error. We now turn to the policy suboptimality gap. We first introduce following lemma.

\begin{lemma}\label{lemma:supVoptgap_leq_supVdiff}
Let $\hat\pi^* \in \Pi$ satisfy
\begin{equation*}
\hat{\mathbf{T}}^{\hat\pi^*}(\hat v)(s)=\hat{\mathbf{T}}^*(\hat v)(s),\qquad \forall s\in\mathcal S.
\end{equation*}
Then for every $s\in\mathcal{S}$, we have
\begin{equation*}
0\leq \sup_{\pi\in\Pi}V_{\mathcal{P}}^{\pi}(s)-V_{\mathcal{P}}^{\hat{\pi}^*}(s)\leq 2\sup_{\pi\in\Pi}\left\|V_{\hat{\mathcal{P}}}^{\pi}-V_{\mathcal{P}}^{\pi}\right\|_{\infty}
\end{equation*}
\end{lemma}

According to Corollary \ref{corollary:value-function-error-2}, Lemma \ref{lemma:supVoptgap_leq_supVdiff} and Proposition \ref{proposition:kl-bellman-operator-error}, we have
\begin{align*}
\sup_{\pi\in\Pi}V_{\mathcal{P}}^{\pi}(s)-V_{\mathcal{P}}^{\hat{\pi}^*}(s)&\leq2\sup_{\pi\in\Pi}\left\|V_{\hat{\mathcal{P}}}^{\pi}-V_{\mathcal{P}}^{\pi}\right\|_{\infty}\\
&\leq\frac{2}{1-\gamma}\sup_{\pi\in\Pi}\left\|\hat{\mathbf{T}}^{\pi}(V_{\mathcal{P}}^{\pi})-\mathcal{T}^{\pi}(V_{\mathcal{P}}^{\pi})\right\|_{\infty}\\
&\leq\frac{18\|R+\gamma v\|_\infty}{\sqrt{n\mathfrak{p}_\wedge}(1-\gamma)}\sqrt{\log{(4|\mathcal{S}|^2|\mathcal{A}|/\eta)}}\\
&\leq\frac{18}{\sqrt{n\mathfrak{p}_\wedge}(1-\gamma)^2}\sqrt{\log{(4|\mathcal{S}|^2|\mathcal{A}|/\eta)}}
\end{align*}
\end{proof}

\subsection{Proof of Lemma \ref{lemma:supVoptgap_leq_supVdiff}}
\begin{proof}
We first show that $\hat\pi^*$ is optimal for the empirical DR-MDP, namely
\begin{equation*}
V_{\hat{\mathcal{P}}}^{\hat\pi^*}(s)=\sup_{\pi\in\Pi}V_{\hat{\mathcal{P}}}^{\pi}(s),\qquad \forall s\in\mathcal S.
\end{equation*}
Since
\begin{equation*}
\hat{\mathbf T}^{\hat\pi^*}(\hat v)(s)=\hat{\mathbf T}^*(\hat v)(s)=\hat v(s), \qquad \forall s\in\mathcal S,
\end{equation*}
the vector $\hat v$ is a fixed point of $\hat{\mathbf T}^{\hat\pi^*}$. Because
$\hat{\mathbf T}^{\hat\pi^*}$ is a $\gamma$-contraction, this fixed point is unique.
By definition, that unique fixed point is $V_{\hat{\mathcal{P}}}^{\hat\pi^*}$. Hence
\begin{equation*}
V_{\hat{\mathcal{P}}}^{\hat\pi^*}=\hat v.
\end{equation*}

Now fix any $\pi\in\Pi$. By definition of $\hat{\mathbf T}^*$, for every $u: S\rightarrow \mathbb{R}$ and every $s\in\mathcal{S}$,
\begin{equation*}
\hat{\mathbf T}^{\pi}(u)(s)\le \hat{\mathbf T}^*(u)(s).
\end{equation*}
In particular,
\begin{equation*}
\hat{\mathbf T}^{\pi}(\hat v)(s)\le \hat{\mathbf T}^*(\hat v)(s)=\hat v(s),\qquad \forall s\in\mathcal{S}.
\end{equation*}
Define the sequence $(u_k)_{k\ge 0}$ by
\begin{equation*}
u_0:=\hat v,\qquad u_{k+1}:=\hat{\mathbf T}^{\pi}(u_k),\quad k\ge 0.
\end{equation*}
Then $u_1(s)\le u_0(s)$ for all $s$. Since $\hat{\mathbf T}^{\pi}$ is monotone, whenever
$u_k(s)\le u_{k-1}(s)$, we have
\begin{equation*}
u_{k+1}=\hat{\mathbf T}^{\pi}(u_k)(s)\le\hat{\mathbf T}^{\pi}(u_{k-1})(s)=u_k(s),\qquad \forall s\in\mathcal{S}.
\end{equation*}
Thus, by induction,
\begin{equation*}
u_k(s)\le \hat v(s),\qquad \forall s\in\mathcal{S},\forall k\ge 1.
\end{equation*}
Since $\hat{\mathbf T}^{\pi}$ is a $\gamma$-contraction, $(u_k)$ converges to the
unique fixed point of $\hat{\mathbf T}^{\pi}$, namely $V_{\hat{\mathcal{P}}}^{\pi}$.
Letting $k\to\infty$ yields
\begin{equation*}
V_{\hat{\mathcal{P}}}^{\pi}(s)\le \hat v(s) = V_{\hat{\mathcal{P}}}^{\hat{\pi}^*}(s),\qquad \forall s\in\mathcal{S}.
\end{equation*}
Because $\pi \in \Pi$ was arbitrary,
\begin{equation*}
V_{\hat{\mathcal{P}}}^{\hat{\pi}^*}(s)=\sup_{\pi\in\Pi}V_{\hat{\mathcal{P}}}^{\pi}(s),\qquad \forall s\in\mathcal{S}.
\end{equation*}
Finally, for any $s\in\mathcal{S}$,
\begin{align*}
0&\leq\sup_{\pi\in\Pi}V_{\mathcal{P}}^{\pi}(s)-V_{\mathcal{P}}^{\hat{\pi}^*}(s)\\
&=\sup_{\pi\in\Pi}V_{\mathcal{P}}^{\pi}(s)-\sup_{\pi\in\Pi}V_{\hat{\mathcal{P}}}^\pi(s)+V_{\hat{\mathcal{P}}}^{\hat{\pi}^*}(s)-V_{\mathcal{P}}^{\hat{\pi}^*}(s)\\
&\leq\left\|\sup_{\pi\in\Pi}V_{\mathcal{P}}^{\pi}-\sup_{\pi\in\Pi}V_{\hat{\mathcal{P}}}^\pi\right\|_{\infty}+\left\|V_{\hat{\mathcal{P}}}^{\hat{\pi}^*}-V_{\mathcal{P}}^{\hat{\pi}^*}\right\|_{\infty}\\
&\leq2\sup_{\pi\in\Pi}\left\|V_{\mathcal{P}}^{\pi}-V_{\hat{\mathcal{P}}}^\pi\right\|_{\infty}
\end{align*}  
This completes the proof.
\end{proof}


\section{Proofs of Properties of the Empirical Bellman Operator: \texorpdfstring{$f$}{f}-Divergence Case}

\subsection{Proof of Proposition \ref{proposition:f-divergence-bellman-operator-error}}
\label{sub:proposition-f-proof}
\begin{proof}
Let 
\begin{equation*}
f(\boldsymbol \mu,u,\boldsymbol\eta)=-c(k,\rho,|\mathcal{A}|)\left(\sum_{a\in\mathcal{A}}\mu_a\left[w_a^{k^*} \right]\right)^{1/k^*}+\sum_{a\in\mathcal{A}} \eta_a,
\end{equation*}
where $w_a = (\eta_a-d_au)_+$. By definition, we have
\begin{align*}
&P\left(\sup_{\pi}\left|\hat{\mathbf{T}}^\pi(v)(s) - \mathcal{T}^\pi(v)(s)\right|>t\right)\\
&\leq P\left(\sup_{d\in\Delta(|\mathcal{A}|)} \left|\sup_{\boldsymbol{\eta}\in\mathbb{R}^{|\mathcal{A}|}}f(\boldsymbol{\mu}_n,R(s,\cdot,\cdot)+\gamma v,\boldsymbol{\eta})-\sup_{\boldsymbol{\eta}\in\mathbb{R}^{|\mathcal{A}|}}f(\boldsymbol{\mu},R(s,\cdot,\cdot)+\gamma v,\boldsymbol{\eta})\right|>t\right).
\end{align*}
We analyze the sensitivity of the mapping $\boldsymbol{\mu}\ra f(\boldsymbol{\mu},u,\lambda)$. To control the difference between the empirical and the population objective, we establish the following lemma. The proof is deferred to Appendix~\ref{sub:f-divergence-Delta-proof}.
\begin{lemma} For any fixed $u$ and $\pi$, there exists $t\in [0,1]$ such that
  \begin{equation*}
  \left|\sup_{\boldsymbol{\eta}\in\mathbb{R}^{|\mathcal{A}|}}f(\boldsymbol{\mu}_n,u,\boldsymbol{\eta})-\sup_{\boldsymbol{\eta}\in\mathbb{R}^{|\mathcal{A}|}}f(\boldsymbol{\mu},u,\boldsymbol{\eta})\right| \leq 2\|u\|_{\infty,\boldsymbol{\mu}}\left\|\frac{dm_{n}}{d\mu_{n}(t)}\right\|_{\infty,\boldsymbol\mu}.
  \end{equation*}
  \label{lemma:f-divergence-Delta}
\end{lemma}
We decomposed the probability using the event $\Omega_{n,\delta_0}(\boldsymbol{\mu})$ where the empirical estimates are close to the population measures. By using lemma \ref{lemma:f-divergence-Delta}, we obtain 
\begin{align*}
&P\left(\sup_{d\in\Delta(|\mathcal{A}|)} \left|\sup_{\boldsymbol{\eta}\in\mathbb{R}^{|\mathcal{A}|}}f(\boldsymbol{\mu}_n,u,\boldsymbol{\eta})-\sup_{\boldsymbol{\eta}\in\mathbb{R}^{|\mathcal{A}|}}f(\boldsymbol{\mu},u,\boldsymbol{\eta})\right|>t\right)\\
&\leq P(\Omega_{n,\delta_0}(\boldsymbol\mu)^c)+P\left(2\left\|u\right\|_{\infty}\left\|\frac{dm_n}{d\mu_n(\tau)}\right\|_{\infty,\boldsymbol\mu}>t,\Omega_{n,\delta_0}(\boldsymbol\mu)\right)
\end{align*}
Again using Lemma \ref{lemma:mut-near-mu}, we have
\begin{align*}
&\leq P\left(\sup_{a\in\mathcal{A},s^\prime\in\mathcal{S}}\left|\frac{\mu_{a,n}(s^\prime)-\mu_a(s^\prime)}{\mu_a(s^\prime)}\right|>\delta_0\right) + P\left(2\|u\|_{\infty,\boldsymbol{\mu}}\sup_{a\in\mathcal{A},s^\prime\in\mathcal{S}}\left|\frac{\mu_{a,n}(s^\prime)-\mu_a(s^\prime)}{(1-\delta_0) \mu_{a}(s^\prime)}\right|>t\right).
\end{align*}
By Chernoff Bound and Bernstein Inequality, we obtain
\begin{align*}
&\leq P\left(\sup_{a\in\mathcal{A},s^\prime\in\mathcal{S}}\left|\frac 1n \sum_{i=1}^n \mathds{1}(S_i=s^\prime)-\mu_a(s^\prime)\right|>\delta_0\mu_a(s^\prime)\right) \\ &+ P\left(\frac{2}{1-\delta_0}\|u\|_{\infty}\sup_{a\in\mathcal{A},s^\prime\in\mathcal{S}}\frac 1{\mu_{a}(s^\prime)}\left|\frac 1n \sum_{i=1}^n \mathds{1}(S_i=s^\prime)-\mu_a(s^\prime)\right|>t\right) \\
&\leq 2\sum_{a\in\mathcal{A}}\sum_{s^\prime\in \mathcal{S}}\left(\exp\left(-\frac{\delta_0^2n\mu_a(s^\prime)}{3}\right)+\exp\left(-\frac{t^2}2\left(\frac{4\|u\|_{\infty}^2}{(1-\delta_0)^2n\mu_a(s^\prime)} + \frac{2\|u\|_{\infty}t}{3(1-\delta_0)n\mu_a(s^\prime)}\right)^{-1}\right)\right),
\end{align*}
Since $\mu_a(s^\prime)\geq \mathfrak{p}_{\wedge}$, and both exponential term above is monotonically decreasing over $\mu_a(s^\prime)$, we have
\begin{align*}
&\leq 2|\mathcal{A}||\mathcal{S}|\left(\exp\left(-\frac{\delta_0^2n\mathfrak{p}_{\wedge}}{3}\right)+\exp\left(-\frac{t^2}2\left(\frac{4\|u\|_{\infty}^2}{(1-\delta_0)^2n\mathfrak{p}_\wedge} + \frac{2\|u\|_{\infty}t}{3(1-\delta_0)n\mathfrak{p}_\wedge}\right)^{-1}\right)\right)
\end{align*}
Choose $\delta_0 = \frac12$, by union bound, we obtain
\begin{align*}
&P\left(\sup_{\pi\in\Pi}\left\|\hat{\mathbf{T}}^\pi(v)-\mathcal{T}^\pi(v)\right\|_\infty>t\right)\\
&\leq P\left(\sup_{s\in\mathcal{S}}\sup_{d\in\Delta(\mathcal{A})}\left|\sup_{\boldsymbol{\eta}\in\mathbb{R}^{|\mathcal{A}|}}f(P_{s,n},R(s,\cdot,\cdot)+\gamma v,\boldsymbol{\eta})-\sup_{\boldsymbol{\eta}\in\mathbb{R}^{|\mathcal{A}|}}f(P_{s},R(s,\cdot,\cdot)+\gamma v,\boldsymbol{\eta})\right|>t\right)\\
&\leq 2|\mathcal{S}|^2|\mathcal{A}|\exp\left(-\frac{n\mathfrak{p}_{\wedge}}{12}\right)+2|\mathcal{S}|^2|\mathcal{A}|\exp\left(-\frac{t^2}{2}\left(\frac{16\|R(s,\cdot,\cdot)+\gamma v\|_{\infty}^2}{n\mathfrak{p}_{\wedge}} + \frac{4\|R(s,\cdot,\cdot)+\gamma v\|_{\infty}t}{3\gamma n\mathfrak{p}_{\wedge}}\right)^{-1}\right).
\end{align*}
Set each term to be less than $\eta/2$, by union bound, we need
\begin{gather}
n\geq \frac{12}{\mathfrak{p}_{\wedge}}\log\left(4|\mathcal{S}|^2|\mathcal{A}|/\eta\right)\label{eq:proposition-f-n}\\
t\geq \frac{8\|R+\gamma v\|_{\infty}}{3n\mathfrak{p}_{\wedge}}\log\left(4|\mathcal{S}|^2|\mathcal{A}|/\eta\right) + \frac{4\|R+\gamma v\|_{\infty}}{\sqrt{n\mathfrak{p}_{\wedge}}}\sqrt{2\log\left(4|\mathcal{S}|^2|\mathcal{A}|/\eta\right)}\label{eq:proposition-f-t}.
\end{gather}
Under \eqref{eq:proposition-f-n}, we have
\begin{equation*}
\frac{\log(4|\cS|^2|\cA|/\eta)}{n\mathfrak{p}_\wedge}\leq \sqrt{\frac{\log(4|\cS|^2|\cA|/\eta)}{n\mathfrak{p}_\wedge}}.
\end{equation*}
By substituting this bound into \eqref{eq:proposition-f-t}, we obtain
\begin{align*}
&\frac{8\|R+\gamma v\|_{\infty}}{3n\mathfrak{p}_{\wedge}}\log\left(4|\mathcal{S}|^2|\mathcal{A}|/\eta\right) + \frac{4\|R+\gamma v\|_{\infty}}{\sqrt{n\mathfrak{p}_{\wedge}}}\sqrt{2\log\left(4|\mathcal{S}|^2|\mathcal{A}|/\eta\right)}\\
&\leq \left(\frac{8}{3}+4\sqrt{2}\right)\frac{\|R+\gamma v\|_{\infty}}{\sqrt{n\mathfrak{p}_{\wedge}}}\sqrt{\log\left(4|\mathcal{S}|^2|\mathcal{A}|/\eta\right)}\\
&\leq \frac{9\|R+\gamma v\|_{\infty}}{\sqrt{n\mathfrak{p}_{\wedge}}}\sqrt{\log\left(4|\mathcal{S}|^2|\mathcal{A}|/\eta\right)}.
\end{align*}
Therefore, when $n$ satisfies \eqref{eq:proposition-f-n} and $t$ satisfies
\begin{equation*}
t\geq \frac{9\|R+\gamma v\|_{\infty}}{\sqrt{n\mathfrak{p}_{\wedge}}}\sqrt{\log\left(4|\mathcal{S}|^2|\mathcal{A}|/\eta\right)},
\end{equation*}
we have
\begin{equation*}
P\left(\sup_{\pi\in\Pi}\left\|\hat{\mathbf{T}}^\pi(v)-\mathcal{T}^\pi(v)\right\|_\infty>t\right) \leq \eta,
\end{equation*}
Since $|\sup_X f-\sup_X g|\leq \sup_X|f-g|$,
\begin{align*}
P\left(\sup_{\pi\in\Pi}\left\|\hat{\mathbf{T}}^*(v)-\mathcal{T}^*(v)\right\|_\infty>t\right)&=P\left(\left\|\sup_{\pi\in\Pi}\hat{\mathbf{T}}^\pi(v)-\sup_{\pi\in\Pi}\mathcal{T}^\pi(v)\right\|_\infty>t\right)\\
&\leq P\left(\sup_{\pi\in\Pi}\left\|\hat{\mathbf{T}}^\pi(v)-\mathcal{T}^\pi(v)\right\|_\infty>t\right) \\
&\leq \eta.
\end{align*}
which implies the statement of the proposition. 
\end{proof}

\subsection{Proof of Lemma \ref{lemma:f-divergence-Delta}}
\label{sub:f-divergence-Delta-proof}
\begin{proof}
We partition $\mathbb{R}^{|\mathcal{A}|}$ into three subsets, denoted by
\begin{align*}
  X_1 &= \left\{\boldsymbol{\eta}\Big|\eta_a\leq d_a\operatorname*{essinf}_{\mu_a}u~\text{for all }a\in\mathcal{A}\right\},\\
  X_2 &= \left\{\boldsymbol{\eta}\Big|d_a\operatorname*{essinf}_{\mu_a}u<\eta_a<M ~\text{for all }a\in\mathcal{A}\right\},\\
  X_3 &= \mathbb{R}^{|\mathcal{A}|}\setminus (X_1\cup X_2).
\end{align*}
where
\begin{equation*}
M =\max_{a,s^\prime}d_au(s^\prime) +(c(k,\rho,|\mathcal{A}|) - |\mathcal{A}|^{1/k})^{-1}\sum_{a\in\mathcal{A}}d_a\max_{s^\prime}u(s^\prime)
\end{equation*}

For any fixed $\boldsymbol{\mu},\boldsymbol{\mu}_n$ and $u$, let
\begin{gather*}
g(\boldsymbol{\eta},t) = f(\boldsymbol{\mu}_n(t),u,\boldsymbol{\eta}),\\
V(t) = \sup_{\boldsymbol{\eta}\in X_2}g(\boldsymbol{\eta},t).
\end{gather*}

First we record the regularity properties of $g$ that will be used to verify
the assumptions of the envelope theorem.

\begin{lemma}
\label{lemma:g-property}
The function $g(\boldsymbol{\eta},t)$ has the following properties:
\begin{itemize}
  \item For every fixed $t\in[0,1]$, the map
  $\boldsymbol{\eta}\mapsto g(\boldsymbol{\eta},t)$ is continuous on $X_2$.
  \item For every fixed $\boldsymbol{\eta}\in X_2$, the map
  $t\mapsto g(\boldsymbol{\eta},t)$ is absolutely continuous and differentiable
  on $[0,1]$.
  \item Under $\Omega_{n,\delta_0}$, there exists a finite constant $b_0$ such that
  \[
  |\partial_t g(\boldsymbol{\eta},t)|\le b_0,
  \qquad \forall \boldsymbol{\eta}\in X_2,\ \forall t\in[0,1].
  \]
\end{itemize}
\end{lemma}

\begin{lemma}
\label{lemma:optimization-structure}
For any fixed measure vector $\boldsymbol{\mu}=(\mu_a)_{a\in\mathcal A},$
consider the optimization problem
\begin{equation*}
\sup_{\boldsymbol{\eta}\in\mathbb{R}^{|\mathcal A|}}f(\boldsymbol{\mu},u,\boldsymbol{\eta})=\sup_{\boldsymbol{\eta}\in\mathbb{R}^{|\mathcal A|}}
-c(k,\rho,|\mathcal A|)\left(\sum_{a\in\mathcal A}\mu_a[w_a^{k^*}]\right)^{1/k^*}+\sum_{a\in\mathcal A}\eta_a,
\end{equation*}
where $w_a=(\eta_a-d_au)_+$.
Then:
\begin{itemize}
  \item this optimization problem admits an optimal solution  $\boldsymbol{\eta}^*(\boldsymbol{\mu})$;
  \item any optimal solution belongs to $X_1\cup X_2$;
  \item the optimal solution satisfies
  \begin{equation}
  \left(\sum_{a\in\mathcal A}\mu_a[w_a^{k^*}]\right)^{1/k}  =  c(k,\rho,|\mathcal A|)  \mu_i[w_i^{1/(k-1)}],  \qquad \forall i\in\mathcal A;
  \label{eq:kl-first-order}
  \end{equation}
  \item if \(\boldsymbol{\eta}^*(\boldsymbol{\mu})\in X_2\), then
  \begin{equation*}
  f(\boldsymbol{\mu},u,\boldsymbol{\eta}^*(\boldsymbol{\mu}))  =  -\frac{\sum_{a\in\mathcal A}\mu_a[w_a^{k^*}]}{\mu_i[w_i^{1/(k-1)}]} +\sum_{a\in\mathcal A}\eta_a^*(\boldsymbol{\mu}),  \qquad \forall i\in\mathcal A.
  \end{equation*}
\end{itemize}
\end{lemma}
The proof of Lemma \ref{lemma:g-property} and Lemma \ref{lemma:optimization-structure} is deferred to Appendix~\ref{sub:proof-of-lemma-g-property} and Appendix~\ref{sub:proof-of-lemma-optimization-structure}, respectively.

When $\boldsymbol{\eta}\in X_1$, we have
\begin{align*}
\left|\sup_{\boldsymbol{\eta}\in X_1}f(\boldsymbol \mu_n,u,\boldsymbol\eta)-\sup_{\boldsymbol{\eta}\in X_1}f(\boldsymbol \mu,u,\boldsymbol\eta)\right|&\leq\sup_{\boldsymbol{\eta}\in X_1}\left|f(\boldsymbol \mu_n,u,\boldsymbol\eta)-f(\boldsymbol \mu,u,\boldsymbol\eta)\right|\\
&=\left|\left(-0+\sum_{a\in\mathcal{A}}\eta_a\right)-\left(-0+\sum_{a\in\mathcal{A}}\eta_a\right)\right| = 0
\end{align*}

Otherwise, $\boldsymbol{\eta}\in X_2$. 
Before proceeding, we introduce the following version of the envelope theorem, which ensures the differentiability of $V(t)$ and provides an explicit formula for its derivative. This result allows us to apply the mean value theorem in the subsequent analysis.

\begin{lemma}[Envelope theorem, \citet{milgrom2002envelope}, Theorem 2]Denote $V$ as
\begin{equation*}
V(t) = \sup_{\mathbf{x}\in X}f(\mathbf{x},t).
\end{equation*}
Suppose that $f(\mathbf{x},\cdot)$ is absolutely continuous for all $\mathbf{x} \in X$. Suppose also that there exists an integrable function $b:[0,1]\rightarrow \mathbb{R}_+$ such that $|\partial_t f(\mathbf{x},t)|\leq b(t)$ for all $\mathbf{x}\in X$ and almost all $t\in[0,1]$. Then $V$ is absolutely continuous. Suppose, in addition, that $f(\mathbf{x},\cdot)$ is differentiable for all $\mathbf{x}\in X$, and that $X^*(t)\neq \varnothing$ almost everywhere on $[0,1]$. Then for any selection $\mathbf{x}^*(t)\in X^*(t)$,
\begin{equation*}
V(t) = V(0) + \int_0^t \partial_tf(\mathbf{x}^*(s),s)ds.
\end{equation*}
\label{lemma:envelope theorem}
\end{lemma}
On the event $\Omega_{n,\delta_0}$, by Lemma \ref{lemma:g-property}, we have that $g(\boldsymbol{\eta},\cdot)$ is absolutely continuous for all $\boldsymbol{\eta}\in X_2$, and there exists an integrable function $b_0$ such that $|\partial_t g(\boldsymbol{\eta},t)|\leq b_0$ for all $\boldsymbol{\eta}\in X_2$ and all $t\in[0,1]$. Then by using Lemma \ref{lemma:envelope theorem}, we have,

\begin{align*}
\left|\sup_{\boldsymbol{\eta}\in X_2}f(\boldsymbol \mu_n,u,\boldsymbol\eta)
-\sup_{\boldsymbol{\eta}\in X_2}f(\boldsymbol \mu,u,\boldsymbol\eta)\right|
&=\left|V(0)-V(1)\right|\\
&=\left|\int_0^1 \frac{d}{dt}g(\boldsymbol{\eta}^*(t),t)\,dt\right| \\
&\le \sup_{t\in[0,1]}\left|\frac{\partial}{\partial t} g(\boldsymbol{\eta}^*(t),t)\right|.
\end{align*}
Next, we bound for any $\tau\in[0,1]$
\begin{align*}
\left|\frac{\partial}{\partial t} g(\boldsymbol{\eta}^*(t),t)\Big|_{t=\tau}\right|
&=\left|\frac{c(k,\rho,|\mathcal{A}|)}{k^*\left(\sum_{a\in\mathcal{A}}\mu_{a,n}(\tau)\left[w_a^{k^*} \right]\right)^{1/k}}\sum_{a\in\mathcal{A}}m_{a,n}\left[w_a^{k^*}\right]\right|\\
&=\left|\sum_{a\in\mathcal{A}}\frac{c(k,\rho,|\mathcal{A}|)}{k^*\left(\sum_{\ell\in\mathcal{A}}\mu_{\ell,n}(\tau)\left[w_\ell^{k^*} \right]\right)^{1/k}}m_{a,n}\left[w_a^{k^*}\right]\right|.
\end{align*}
By using (\ref{eq:kl-first-order}), we have
\begin{equation*}
\frac{c(k,\rho,|\mathcal{A}|)}{k^*\left(\sum_{a\in\mathcal{A}}\mu_{a,n}(t)\left[w_a^{k^*} \right]\right)^{1/k}} = \frac{1}{k^*\cdot\mu_{a,n}(t)\left[w_a^{1/(k-1)}\right]}, \qquad \forall a\in\mathcal{A}.
\end{equation*}
Therefore, combining the previous displays,
\begin{align*}
\left|\frac{\partial}{\partial t} g(\boldsymbol{\eta}^*(t),t)\Big|_{t=\tau}\right|&=\left|\sum_{a\in\mathcal{A}}\frac{m_{a,n}\left[w_a^{k^*}\right]}{k^*\cdot\mu_{a,n}(\tau)\left[w_a^{1/(k-1)}\right]} \right|\\
&\leq\sum_{a\in\mathcal{A}}\left|\frac{m_{a,n}\left[w_a^{k^*}\right]}{k^*\cdot\mu_{a,n}(\tau)\left[w_a^{1/(k-1)}\right]} \right|.
\end{align*}
For $\eta_a \geq d_a\operatorname*{essinf}_{\mu_a}u$, by mean value theorem, there exists $\xi \in (\eta_a-d_a\|u\|_{L^{\infty}(\mu_a)},\eta_a-d_au)$ satisfies
\begin{align*}
\left|\frac{m_{a,n}\left[w_a^{k^*}\right]}{k^*\cdot\mu_{a,n}(\tau)\left[w_a^{1/(k-1)}\right]} \right|&\stackrel{(i)}{=}\left|\frac{m_{a,n}\left[\left(\eta_a-d_au\right)_+^{k^*}-(\eta_a-d_a\|u\|_{L^{\infty}(\mu_a)})_+^{k^*}\right]}{\mu_{a,n}(\tau)\left[k^*\left(\eta_a-d_au\right)_+^{1/(k-1)}\right]}\right|\\
&=\left|\frac{m_{a,n}\left[d_a(u-\|u\|_{L^{\infty}(\mu_a)})k^*(\xi)_+^{1/(k-1)}\right]}{\mu_{a,n}(\tau)\left[k^*\left(\eta_a-d_au\right)_+^{1/(k-1)}\right]}\right|,
\end{align*}
where $(i)$ follows from the fact that $m_{a,n}[c] = 0$ for any constant function $c$.
Since $\xi<\eta_a-d_au$, we have
\begin{align*}
&\left|\frac{m_{a,n}\left[d_a(u-\|u\|_{L^{\infty}(\mu_a)})k^*(\xi)_+^{1/(k-1)}\right]}{\mu_{a,n}(\tau)\left[k^*\left(\eta_a-d_au\right)_+^{1/(k-1)}\right]}\right|\\
&\leq \left\|\frac{d_a(u-\|u\|_{L^{\infty}(\mu_a)})k^*(\xi)_+^{1/(k-1)}}{k^*\left(\eta_a-d_au\right)_+^{1/(k-1)}}\right\|_{L^\infty(\mu_a)}\left\|\frac{dm_{n}}{d\mu_{n}(\tau)}\right\|_{L^\infty(\mu_a)}\\
&\leq\left\|\frac{d_a(u-\|u\|_{L^{\infty}(\mu_a)})(\eta_a-d_au)_+^{1/(k-1)}}{\left(\eta_a-d_au\right)_+^{1/(k-1)}}\right\|_{L^\infty(\mu_a)}\left\|\frac{dm_{n}}{d\mu_{n}(\tau)}\right\|_{L^\infty(\mu_a)}\\
&= 2d_a\|u\|_{L^\infty(\mu_a)}\left\|\frac{dm_{n}}{d\mu_{n}(\tau)}\right\|_{L^\infty(\mu_a)}.
\end{align*}
Therefore, we have
\begin{align*}
\left|\frac{\partial}{\partial t} g(\boldsymbol{\eta}^*(t),t)\Big|_{t=\tau}\right|&=\sum_{a\in\mathcal{A}}\left|\frac{m_{a,n}\left[w_a^{k^*}\right]}{k^*\cdot\mu_{a,n}(\tau)\left[w_a^{1/(k-1)}\right]} \right|\\
&\leq \sum_{a\in\mathcal{A}}2d_a\|u\|_{L^\infty(\mu_a)}\left\|\frac{dm_{n}}{d\mu_{n}(\tau)}\right\|_{L^\infty(\mu_a)}\\
&=2\|u\|_{\infty}\left\|\frac{dm_{n}}{d\mu_{n}(\tau)}\right\|_{\infty,\boldsymbol{\mu}}.
\end{align*}

To conclude the lemma, we we observe that $\boldsymbol{\eta}\in X_1$,
\begin{equation*}
\left|\sup_{\boldsymbol{\eta}\in X_1}f(\boldsymbol \mu_n,u,\boldsymbol\eta)-\sup_{\boldsymbol{\eta}\in X_1}f(\boldsymbol \mu,u,\boldsymbol\eta)\right| = 0.
\end{equation*}
When $\boldsymbol{\eta}\in X_2$,
\begin{equation*}
\begin{aligned}
\left|\sup_{\boldsymbol{\eta}\in X_2}f(\boldsymbol \mu_n,u,\boldsymbol\eta)-\sup_{\boldsymbol{\eta}\in X_2}f(\boldsymbol \mu,u,\boldsymbol\eta)\right| &\leq\sup_{t\in[0,1]}\left|\frac{\partial}{\partial t} g(\boldsymbol{\eta}^*(t),t)\right|\\
&\leq \sup_{t\in[0,1]} 2\|u\|_{\infty}\left\|\frac{dm_{n}}{d\mu_{n}(t)}\right\|_{\infty,\boldsymbol{\mu}}\\
&=  2\|u\|_{\infty}\left\|\frac{dm_{n}}{d\mu_{n}(\tau)}\right\|_{\infty,\boldsymbol{\mu}}
\end{aligned}
\end{equation*}
for some $\tau\in[0,1]$, where the existence of $\tau$ follows from the compactness of $[0,1]$ and the continuity of $t\ra \mu_n(t)$ and norms. 

Finally, using $|\max\lbrace C,x\rbrace-\max\lbrace C,y\rbrace|\leq |x-y|$, we obtain
\begin{equation*}
\left|\sup_{\boldsymbol{\eta}\in X_1\cup X_2}f(\boldsymbol \mu_n,u,\boldsymbol\eta)-\sup_{\boldsymbol{\eta}\in X_1\cup X_2}f(\boldsymbol \mu,u,\boldsymbol\eta)\right|\leq\left|\sup_{\boldsymbol{\eta}\in X_2}f(\boldsymbol \mu_n,u,\boldsymbol\eta)-\sup_{\boldsymbol{\eta}\in X_2}f(\boldsymbol \mu,u,\boldsymbol\eta)\right|\leq 2\|u\|_{\infty}\left\|\frac{dm_{n}}{d\mu_{n}(\tau)}\right\|_{\infty,\boldsymbol{\mu}}
\end{equation*}
\end{proof}

\subsection{Proof of Lemma \ref{lemma:g-property}}
\label{sub:proof-of-lemma-g-property}
Let 
\begin{equation*}
\Phi(\boldsymbol{\eta},t) := \sum_{a\in\mathcal{A}}\mu_{a,n}(t)\left[(\eta_a - d_au)_+^{k^*}\right].
\end{equation*}
For every fixed $t\in[0,1]$, since $u\in L^\infty$ and $\eta_a < M$ on $X_2$, the map 
\begin{equation*}
\eta_a \mapsto  (\eta_a - d_au)_+
\end{equation*}
is uniformly Lipschitz on $(-\infty,M]$. Hence $\Phi(\cdot,t)$ is continuous on $X_2$, and therefore $g(\cdot,t)$ is continuous on $X_2$.

Now fix $\boldsymbol{\eta}\in X_2$. Then
\begin{equation*}
\Phi(\boldsymbol{\eta},t) = \sum_{a\in\mathcal{A}} \left(t\mu_{a}\left[w_a^{k^*}\right] + (1-t)\mu_{a,n}\left[w_a^{k^*}\right]\right),
\end{equation*}
so $\Phi(\boldsymbol{\eta}, \cdot)$ is affine in t. Moreover, under $\Omega_{n,\delta_0}$, since $\eta_a > d_a \operatorname*{essinf}_{\mu_a} u$, we have $\mu_{a}\left[w_a^{k^*}\right] > 0$, and Lemma \ref{lemma:mut-near-mu} yields
\begin{equation*}
\mu_{a,n}\left[w_a^{k^*}\right] \geq (1-\delta_0)\mu_{a}\left[w_a^{k^*}\right] > 0,\qquad \forall a\in\mathcal{A}, \forall t\in[0,1].
\end{equation*}
Therefore $\Phi(\boldsymbol{\eta},t) > 0$ for all $t\in[0,1]$. Since $x\mapsto x^{1/k^*}$ is $C^1$ on $(0,\infty)$, it follows that $t\mapsto g(\boldsymbol{\eta},t)$ is $C^1$ on $[0,1]$, hence absolutely continuous and differentiable on $[0,1]$.

Finally we check under $\Omega_{n,\delta_0}$, when $\boldsymbol \eta \in X_2$, $\partial_t g(\boldsymbol \eta, t)$ has an upper bound. Recall that
\begin{equation*}
\partial_t g(\boldsymbol \eta, t) = -c(k,\rho,|\mathcal{A}|)\frac{\sum_{a\in\mathcal{A}}m_{a,n}\left[w_a^{k^*}\right]}{k^*\left(\sum_{a\in\mathcal{A}}\mu_{a,n}(t)\left[w_a^{k^*} \right]\right)^{1/k}}.
\end{equation*}

Let 
\begin{equation*}
W:= \max_{a\in\mathcal{A}} \left(M - d_a\operatorname*{essinf}_{\mu_a}u\right)\leq \infty,
\end{equation*}
by the definition of $w_a$, for every $a\in\mathcal A$ and $s$,
\begin{equation*}
w_a(s) = (\eta_a - d_au(s))_+\leq M - d_a\operatorname*{essinf}_{\mu_a}u \leq W.
\end{equation*}
Applying Lemma \ref{lemma:mut-near-mu}, we have
\begin{equation*}
\left|m_{a,n}\left[w_a^{k^*}\right]\right| = \left|\mu_{a,n}\left[w_a^{k^*}\right]-\mu_a\left[w_a^{k^*}\right] \right|\leq \delta_0\mu_{a,n}\left[w_a^{k^*}\right]\leq \frac{\delta_0}{1-\delta_0} \mu_{a,n}(t)\left[w_a^{k^*}\right].
\end{equation*}

Therefore,
\begin{align*}
|\partial_t g(\boldsymbol{\mu},t) | &\leq \frac{c(k,\rho,|\mathcal{A}|)}{k^*}\frac{\delta_0}{1-\delta_0}\frac{\sum_{a\in\mathcal{A}}\mu_{a,n}(t)\left[w_a^{k^*}\right] }{\left(\sum_{a\in\mathcal{A}}\mu_{a,n}(t)\left[w_a^{k^*} \right]\right)^{1/k}}\\
&=\frac{c(k,\rho,|\mathcal{A}|)}{k^*}\frac{\delta_0}{1-\delta_0}\left(\sum_{a\in\mathcal{A}}\mu_{a,n}(t)\left[w_a^{k^*} \right]\right)^{1/k^*}\\
&\leq \frac{c(k,\rho,|\mathcal{A}|)}{k^*}\frac{\delta_0}{1-\delta_0}|\mathcal{A}|^{1/k^*}W < \infty.
\end{align*}

$|\partial_t g(\boldsymbol{\mu},t) |$ is bounded. 

\subsection{Proof of Lemma \ref{lemma:optimization-structure}}
\label{sub:proof-of-lemma-optimization-structure}
We prove the lemma in several steps.
\par
\noindent\textbf{Step 1: Existence of an optimizer.}
We choose a closed set
\begin{equation*}
D:= \left\{\boldsymbol{\eta}\in\mathbb{R}^{|\mathcal{A}|}: \eta_a \geq d_a\operatorname*{essinf}_{\mu_a} u \right\}.
\end{equation*}
For any $\boldsymbol{\eta} \in \mathbb{R}^{|\mathcal{A}|}$, let $\eta_a = \max\{\eta_a, d_a\operatorname*{essinf}_{\mu_a} u\}$. Since $(\eta_a - d_au)_+=0$ for $\eta_a \leq \operatorname*{essinf}_{\mu_a} u\}$,
\begin{equation*}
f(\boldsymbol{\mu},u,\tilde{\boldsymbol{\eta}}) \geq f(\boldsymbol{\mu},u,\boldsymbol{\eta}).
\end{equation*}
Therefore
\begin{equation*}
\sup_{\boldsymbol{\eta}\in\mathbb{R}^{|\mathcal{A}|}}f(\boldsymbol{\mu},u,\boldsymbol{\eta}) = \sup_{\boldsymbol{\eta}\in D}f(\boldsymbol{\mu},u,\boldsymbol{\eta}).
\end{equation*}
Next we prove that $f$ is coercive on $D$. Let 
\begin{equation*}
U_a := d_a \operatorname*{esssup}_{\mu_a} u,
\end{equation*}
then we have
\begin{equation*}
(\mu_a - d_au(s^\prime))_+ \geq (\mu_a - U_a)_+ \qquad \forall s^\prime\in\mathcal{S},
\end{equation*}
therefore
\begin{equation*}
\mu_a\left[(\eta_a - d_au)_+^{k^*}\right] \geq (\eta_a - U_a)_+^{k^*}.
\end{equation*}
Then we have
\begin{align*}
f(\boldsymbol{\mu},u,\boldsymbol{\eta})&=-c(k,\rho,|\mathcal{A}|)\left(\sum_{a\in\mathcal{A}}\mu_a\left[(\eta_a - d_au)_+^{k^*}\right]\right)^{1/k^*}+\sum_{a\in\mathcal{A}} \eta_a\\
&\leq -c(k,\rho,|\mathcal{A}|)\left(\sum_{a\in\mathcal{A}}(\eta_a - U_a)_+^{k^*}\right)^{1/k^*}+\sum_{a\in\mathcal{A}} \eta_a.
\end{align*}
Since $\Vert x\Vert_k^* \geq |\mathcal{A}|^{-1/k}\Vert x\Vert_1$,
\begin{align*}
f(\boldsymbol{\mu},u,\boldsymbol{\eta})&\leq -c(k,\rho,|\mathcal{A}|)|\mathcal{A}|^{-1/k}\sum_{a\in\mathcal{A}}(\eta_a - U_a)_+ + \sum_{a\in\mathcal{A}} \eta_a\\
&\leq -c(k,\rho,|\mathcal{A}|)|\mathcal{A}|^{-1/k}\sum_{a\in\mathcal{A}}(\eta_a - U_a) + \sum_{a\in\mathcal{A}} \eta_a\\
&=\left(1 - c(k,\rho,|\mathcal{A}|)|\mathcal{A}|^{-1/k}\right)\sum_{a\in\mathcal{A}} \eta_a + c(k,\rho,|\mathcal{A}|)|\mathcal{A}|^{-1/k}\sum_{a\in\mathcal{A}} U_a.
\end{align*}
Since $\rho > 0$, we have $c(k,\rho,|\mathcal{A}|) > |\mathcal{A}|^{1/k}$, therefore $1 - c(k,\rho,|\mathcal{A}|)|\mathcal{A}|^{-1/k} < 0$. Therefore, when $\Vert \boldsymbol{\eta}\Vert \to \infty$, $f(\boldsymbol{\mu},u,\boldsymbol{\eta}) \to -\infty$. Therefore, $f$ is coercive on $D$. Since $f$ is continuous on $D$, there exists $\boldsymbol{\eta}^* \in D$ such that
\begin{equation*}
\sup_{\boldsymbol{\eta}\in\mathbb{R}^{|\mathcal{A}|}}f(\boldsymbol{\mu},u,\boldsymbol{\eta}) = \sup_{\boldsymbol{\eta}\in D}f(\boldsymbol{\mu},u,\boldsymbol{\eta}) = f(\boldsymbol{\mu},u,\boldsymbol{\eta}^*).
\end{equation*}

\noindent\textbf{Step 2: First-order characterization of an optimal solution.}
Define
\begin{equation*}
\Phi(\boldsymbol{\eta})
:=
\sum_{a\in\mathcal{A}}
\mu_a\!\left[(\eta_a-d_au)_+^{k^*}\right].
\end{equation*}
We distinguish two cases.

\smallskip
\noindent\emph{Case 1: $\Phi(\boldsymbol{\eta}^*)=0$.}
Since each term in the sum is nonnegative, we have
\begin{equation*}
\mu_a\!\left[(\eta_a^*-d_au)_+^{k^*}\right]=0,
\qquad \forall a\in\mathcal{A}.
\end{equation*}
Hence
\begin{equation*}
(\eta_a^*-d_au)_+=0 \qquad \mu_a\text{-a.s.},
\end{equation*}
which implies
\begin{equation*}
\eta_a^* \le d_a\operatorname*{essinf}_{\mu_a}u,
\qquad \forall a\in\mathcal{A}.
\end{equation*}
On the other hand, by Step 1 we already know that $\boldsymbol{\eta}^*\in D$, so
\begin{equation*}
\eta_a^* \ge d_a\operatorname*{essinf}_{\mu_a}u,
\qquad \forall a\in\mathcal{A}.
\end{equation*}
Therefore,
\begin{equation*}
\eta_a^* = d_a\operatorname*{essinf}_{\mu_a}u,
\qquad \forall a\in\mathcal{A},
\end{equation*}
and hence $\boldsymbol{\eta}^*\in X_1$. In this case,
\begin{equation*}
\mu_i[w_i^{1/(k-1)}]=0,
\qquad \forall i\in\mathcal{A},
\end{equation*}
so \eqref{eq:kl-first-order} holds trivially.

\smallskip
\noindent\emph{Case 2: $\Phi(\boldsymbol{\eta}^*)>0$.}
We first show that
\begin{equation*}
\eta_i^* > d_i\operatorname*{essinf}_{\mu_i}u,
\qquad \forall i\in\mathcal{A}.
\end{equation*}
Suppose, to the contrary, that for some $i\in\mathcal{A}$,
\begin{equation*}
\eta_i^* = d_i\operatorname*{essinf}_{\mu_i}u.
\end{equation*}
Then
\begin{equation*}
(\eta_i^*-d_iu)_+=0
\qquad \mu_i\text{-a.s.},
\end{equation*}
and therefore
\begin{equation*}
\mu_i[w_i^{1/(k-1)}]=0.
\end{equation*}
Since $\Phi(\boldsymbol{\eta}^*)>0$, the function $f(\boldsymbol{\mu},u,\cdot)$ is differentiable at $\boldsymbol{\eta}^*$, and
\begin{equation*}
\frac{\partial}{\partial\eta_i}
f(\boldsymbol{\mu},u,\boldsymbol{\eta}^*)
=
1-c(k,\rho,|\mathcal{A}|)\left(\sum_{a\in\mathcal{A}}\mu_a[w_a^{k^*}]\right)^{-1/k}\mu_i[w_i^{1/(k-1)}]
=1.
\end{equation*}
Moreover, since $\boldsymbol{\eta}^*\in D$, the direction $+e_i$ is feasible. Hence, for all sufficiently small $t>0$,
\begin{equation*}
f(\boldsymbol{\mu},u,\boldsymbol{\eta}^*+te_i)
>
f(\boldsymbol{\mu},u,\boldsymbol{\eta}^*),
\end{equation*}
which contradicts the optimality of $\boldsymbol{\eta}^*$. Therefore,
\begin{equation*}
\eta_i^* > d_i\operatorname*{essinf}_{\mu_i}u,
\qquad \forall i\in\mathcal{A}.
\end{equation*}
In particular, $\boldsymbol{\eta}^*\in X_2$, so $\boldsymbol{\eta}^*$ is an interior point of $D$.

Since $f(\boldsymbol{\mu},u,\cdot)$ is concave in $\boldsymbol{\eta}$, we may now apply the first-order optimality condition at the interior maximizer $\boldsymbol{\eta}^*$, which gives
\begin{equation*}
\frac{\partial}{\partial\eta_i}
f(\boldsymbol{\mu},u,\boldsymbol{\eta}^*)
=
1-c(k,\rho,|\mathcal{A}|)\left(\sum_{a\in\mathcal{A}}\mu_a[w_a^{k^*}]\right)^{-1/k}\mu_i[w_i^{1/(k-1)}]
=0,
\qquad \forall i\in\mathcal{A}.
\end{equation*}
Equivalently,
\begin{equation*}
\left(\sum_{a\in\mathcal{A}}\mu_a[w_a^{k^*}]\right)^{1/k}
=
c(k,\rho,|\mathcal{A}|)\mu_i[w_i^{1/(k-1)}],
\qquad \forall i\in\mathcal{A},
\end{equation*}
which is exactly \eqref{eq:kl-first-order}.

\par
\noindent\textbf{Step 3: Value of the objective when $\boldsymbol{\eta}^*(\boldsymbol{\mu})\in X_2$.}
When $\boldsymbol{\eta}^*(\boldsymbol{\mu})\in X_2$, we have
\begin{equation*}
\mu_i[w_i^{1/(k-1)}]>0,
\qquad \forall i\in\mathcal{A}.
\end{equation*}
Substituting \eqref{eq:kl-first-order} into the definition of $f$, we obtain
\begin{align*}
f(\boldsymbol{\mu},u,\boldsymbol{\eta}^*)
&=
-c(k,\rho,|\mathcal{A}|)\left(\sum_{a\in\mathcal{A}}\mu_a[w_a^{k^*}]\right)^{1/k^*}
+\sum_{a\in\mathcal{A}}\eta_a^* \\
&=
-\frac{\sum_{a\in\mathcal{A}}\mu_a[w_a^{k^*}]}
{\mu_i[w_i^{1/(k-1)}]}
+\sum_{a\in\mathcal{A}}\eta_a^*,
\qquad \forall i\in\mathcal{A}.
\end{align*}

\noindent\textbf{Step 4: Any optimizer belongs to $X_1 \cup X_2$.}

We first prove that 
\begin{equation*}
\boldsymbol{\eta}^* \notin X_{3,1}:=\left\{\boldsymbol{\eta}\Big| \exists a_1,a_2\in\mathcal{A}, \eta_{a_1}>d_{a_1}\operatorname*{essinf}_{\mu_{a_1}}u, \eta_{a_2}\leq d_{a_2}\operatorname*{essinf}_{\mu_{a_2}}u \right\}
\end{equation*}
by contradictory.

Suppose $\boldsymbol{\eta}^{*} \in X_{3,1}$. Then, according to the definition of $X_{3,1}$, $\mu_{a_2}[w_{a_2}^{1/(k-1)}] = 0$.  According to \eqref{eq:kl-first-order}, 
\begin{equation*}
\left(\sum_{a\in\mathcal{A}}\mu_a\left[w_a^{k^*}\right]  \right)^{1/k} = 0,
\end{equation*} 
which implies 
\begin{equation*}
\mu_a\left[w_a^{k^*}\right] = 0 \quad \text{for all }a\in\mathcal{A}.
\end{equation*}
Therefore, $\boldsymbol{\eta}^* \in X_1$, which contradicts the assumption that $\boldsymbol{\eta}^* \in X_{3,1}$. Next, we prove that
\begin{equation*}
\boldsymbol{\eta}^* \notin X_{3,2}:=\left\{\boldsymbol{\eta}\Big| \exists a\in\mathcal{A}, \eta_a\geq M \right\}
\end{equation*}
where
\begin{equation*}
M =\max_{a,s^\prime}d_au(s^\prime) +\frac{\sum_{a\in\mathcal{A}}d_a\max_{s^\prime} u(s^\prime)}{c(k,\rho,|\mathcal{A}|) - |\mathcal{A}|^{1/k}}.
\end{equation*}

We first bound the first term in $f$
\begin{equation*}
(\eta_a-d_au(s^\prime))_+ \geq \left(\eta_a - \max_{s^\prime}d_au(s^\prime)\right)_+
\end{equation*}
taking expectation and sum over $a$ on both side of the inequality
\begin{equation*}
\sum_{a\in\mathcal{A}}\mu_a\left[(\eta_a-d_au)^{k^*}_+\right] \geq \sum_{a\in\mathcal{A}}\left(\eta_a - \max_{s^\prime}d_au(s^\prime)\right)_+^{k^*}.
\end{equation*}
Therefore,
\begin{equation*}
\left(\sum_{a\in\mathcal{A}}\mu_a\left[w_a^{k^*}\right] \right)^{1/k^*} \geq \left(\sum_{a\in\mathcal{A}}\left(\eta_a - \max_{s^\prime}d_au(s^\prime)\right)_+^{k^*}\right)^{1/k^*}.
\end{equation*}

Then we bound the second term in $f$
\begin{equation*}
\sum_{a\in\mathcal{A}} \eta_a \leq \sum_{a\in\mathcal{A}} (\eta_a-\max_{s^\prime}d_au(s^\prime))_+ + \sum_{a\in\mathcal{A}}\max_{s^\prime}d_au(s^\prime).
\end{equation*}
According to H\"{o}lder's inequality, we have
\begin{equation*}
\sum_{a\in\mathcal{A}} (\eta_a-\max_{s^\prime}d_au(s^\prime))_+ \leq |\mathcal{A}|^{1/k}\left(\sum_{a\in\mathcal{A}}(\eta_a - \max_{s^\prime}d_au(s^\prime))_+^{k^*}\right)^{1/k^*}.
\end{equation*}
Therefore,
\begin{equation*}
\sum_{a\in\mathcal{A}} \eta_a \leq |\mathcal{A}|^{1/k}\left(\sum_{a\in\mathcal{A}}(\eta_a - \max_{s^\prime}d_au(s^\prime))_+^{k^*}\right)^{1/k^*} + \sum_{a\in\mathcal{A}}\max_{s^\prime}d_au(s^\prime).
\end{equation*}

To sum up, we have
\begin{align*}
f(\boldsymbol \mu,u,\boldsymbol\eta)&=-c(k,\rho,|\mathcal{A}|)\left(\sum_{a\in\mathcal{A}}\mu_a\left[w_a^{k^*} \right]\right)^{1/k^*}+\sum_{a\in\mathcal{A}} \eta_a\\
&\leq -\left(c(k,\rho,|\mathcal{A}|)-|\mathcal{A}|^{1/k}\right)\left(\sum_{a\in\mathcal{A}}\left(\eta_a - \max_{s^\prime}d_au(s^\prime)\right)_+^{k^*}\right)^{1/k^*} + \sum_{a\in\mathcal{A}}\max_{s^\prime}d_au(s^\prime).
\end{align*}
Note that when $\rho >0$,
\begin{equation*}
c(k,\rho,|\mathcal{A}|)-|\mathcal{A}|^{1/k}= |\mathcal{A}|^{1/k}\left(\left(k(k-1)\rho+1 \right)^{1/k}-1\right) > 0
\end{equation*}

Suppose there exists $a_1$ such that $\eta_{a_1}> M$, then
\begin{align*}
f(\boldsymbol \mu,u,\boldsymbol\eta)& \leq -\left(c(k,\rho,|\mathcal{A}|)-|\mathcal{A}|^{1/k}\right)\left(\eta_{a_1} - \max_{s^\prime}d_{a_1}u(s^\prime)\right)_+ + \sum_{a\in\mathcal{A}}\max_{s^\prime}d_au(s^\prime)\\
&< -\left(c(k,\rho,|\mathcal{A}|)-|\mathcal{A}|^{1/k}\right)\left(\frac{\sum_{a\in\mathcal{A}}d_a\max_{s^\prime} u(s^\prime)}{c(k,\rho,|\mathcal{A}|) - |\mathcal{A}|^{1/k}}\right) + \sum_{a\in\mathcal{A}}\max_{s^\prime}d_au(s^\prime)\\
&< 0
\end{align*}

Therefore, $\boldsymbol \eta$ can not be the optimal solution when $\boldsymbol \eta \in X_{3,2}$. 
It can be easily check that $X_3 = X_{3,1}\cup X_{3,2}$, and overall, we have $\boldsymbol{\eta}^* \notin X_3$.

\subsection{Proof of Theorem \ref{theorem:f}}
\label{sub:theorem-f-proof}
\begin{proof}

Substituting $\|R+\gamma v\|_{\infty} \leq 1/(1-\gamma)$ into the bound from Proposition~\ref{proposition:f-divergence-bellman-operator-error} and applying Proposition~\ref{proposition:value-function-error}, we obtain the stated result.
\begin{align*}
\|\hat v - v^*\|_{\infty}&\leq \frac{1}{1-\gamma}\left\|\hat{\mathbf{T}}^*(v^*)-\mathcal{T}^*(v^*)\right\|_\infty\\
&\leq \frac{9\|R+\gamma v\|_{\infty}}{(1-\gamma)\sqrt{n\mathfrak{p}_{\wedge}}}\sqrt{\log\left(4|\mathcal{S}|^2|\mathcal{A}|/\eta\right)}\\
&\leq \frac{9}{(1-\gamma)^2\sqrt{n\mathfrak{p}_{\wedge}}}\sqrt{\log\left(4|\mathcal{S}|^2|\mathcal{A}|/\eta\right)}
\end{align*}
with probability $1-\eta$.
The preceding step provides a uniform bound on the value-function estimation error. We now turn to the policy suboptimality gap. 
According to Corollary \ref{corollary:value-function-error-2}, Lemma \ref{lemma:supVoptgap_leq_supVdiff} and Proposition \ref{proposition:f-divergence-bellman-operator-error}, we have
\begin{align*}
\sup_{\pi\in\Pi}V_{\mathcal{P}}^{\pi}(s)-V_{\mathcal{P}}^{\hat{\pi}^*}(s)&\leq2\sup_{\pi\in\Pi}\left\|V_{\hat{\mathcal{P}}}^{\pi}-V_{\mathcal{P}}^{\pi}\right\|_{\infty}\\
&\leq\frac{2}{1-\gamma}\sup_{\pi\in\Pi}\left\|\hat{\mathbf{T}}^{\pi}(V_{\mathcal{P}}^{\pi})-\mathcal{T}^{\pi}(V_{\mathcal{P}}^{\pi})\right\|_{\infty}\\
&\leq\frac{18\|R+\gamma v\|_{\infty}}{\sqrt{n\mathfrak{p}_\wedge}(1-\gamma)}\sqrt{\log\left(4|S|^2|\mathcal{A}|/\eta\right)}\\
&\leq\frac{18}{\sqrt{n\mathfrak{p}_\wedge}(1-\gamma)^2}\sqrt{\log\left(4|S|^2|\mathcal{A}|/\eta\right)}
\end{align*}
\end{proof}

\end{document}